\documentclass[12pt,oneside,english]{amsart}
\usepackage[T1]{fontenc}
\usepackage[utf8]{inputenc}\usepackage{geometry}
\usepackage{amsthm}
\usepackage{amstext}
\usepackage{amssymb}
\usepackage{mathdots}
\usepackage{thm-restate}
\usepackage{esint}
\usepackage{xcolor}
\usepackage{caption}
\usepackage{subcaption}
\usepackage{graphicx}
\usepackage{array}
\usepackage{booktabs}
\usepackage{float}
\usepackage[numbers]{natbib}
\usepackage[foot]{amsaddr}

\makeatletter

\DeclareMathOperator*{\ev}{\mathbb{E}}

\newcommand{\tkernel}{\mathit{\lambda}}

\newcommand{\argmax}{\mathrm{argmax}}

\renewcommand{\Im}{\mathrm{Im}}

\numberwithin{equation}{section}

\newtheorem{lemma}{Lemma}

\def\deval{d} 
\def\adim{m}

\makeatother

\begin{document}
\title{Deep Hedging Under Non-Convexity:\\ Limitations and a Case for AlphaZero}

\author{Matteo Maggiolo$^1$}
\address{$^1$Vigilant Analytics; Lugano; Switzerland}
\email{matteo.maggiolo@vigilant-analytics.ch}

\author{Giuseppe Nuti$^2$}
\address{$^2$UBS Investment Bank; New York; USA}
\email{giuseppe.nuti@ubs.com}

\author{Miroslav \v{S}trupl$^3$}
\address{$^3$Dalle Molle Institute for Artificial Intelligence (IDSIA) - SUPSI/USI; Lugano; Switzerland}
\email{miroslav.strupl@idsia.ch}

\author{Oleg Szehr$^3$}
\email{oleg.szehr@idsia.ch}

\date{December 6, 2025}%

\begin{abstract}
This paper examines replication portfolio construction in incomplete markets - a key problem in financial engineering with applications in pricing, hedging, balance sheet management, and energy storage planning. We model this as a two-player game between an investor and the market, where the investor makes strategic bets on future states while the market reveals outcomes. Inspired by the success of Monte Carlo Tree Search in stochastic games, we introduce an AlphaZero-based system and compare its performance to deep hedging - a widely used industry method based on gradient descent. 
Through theoretical analysis and experiments, we show that deep hedging struggles in environments where the optimal action-value function is not subject to convexity constraints - such as those involving non-convex transaction costs, capital constraints, or regulatory limitations - converging to local optima. We construct specific market environments to highlight these limitations and demonstrate that AlphaZero consistently finds near-optimal replication strategies. On the theoretical side, we establish a connection between deep hedging and convex optimization, suggesting that its effectiveness is contingent on convexity assumptions. Our experiments further suggest that AlphaZero is more sample-efficient - an important advantage in data-scarce, overfitting-prone derivative markets.
\end{abstract}
\maketitle

\section{Introduction}

Financial markets are inherently complex and unpredictable, making the construction of strategies to mitigate the risk of complex financial assets a key challenge in financial engineering. Replication portfolios allow investors to replicate the return profile of a complex asset using simpler instruments, e.g.~to offset future contractual payoffs, manage exposures or price the asset. Common areas of application include risk management of derivative contracts~\cite{buehler2019deep}, balance sheet management~\cite{krablicher1}, and the planning of energy consumption and storage~\cite{Curin2021}. In practice, replication problems often arise in incomplete markets, where available financial instruments cannot fully replicate the payoffs of complex assets. Incompleteness often arises due to the presence of transaction costs, stochastic market volatility, jump-diffusion dynamics, or regulatory limitations. Under such conditions, investors face the challenge of holding risky positions while balancing replication accuracy against associated costs. To address this, the construction of replication portfolios is often formulated as a stochastic multi-period utility optimization problem~\cite{duffie2001}: the investor consecutively rebalances the holdings in the replication portfolio to maximize the expected utility of terminal wealth (which reflects the discrepancy between the asset's payoffs and the replication portfolio). Although the specific setup varies, such problems have traditionally been phrased in the language of dynamic programming~\cite{hodges1989option,Barron1990,ElKaroui1995,Zakamouline2006} or, more recently, Reinforcement Learning (RL)~\cite{buehler2019deep,buehler2019deep2,cao2019deep,szehrcannelli,Szehr2023}. 

\textit{Deep hedging} (DH), a pioneering machine learning approach introduced in~\cite{buehler2019deep} and widely adopted in industry applications, models the sequential decision-making process using a stack of shallow neural networks (NNs), or alternatively, a recurrent NN or long short-term memory network. Each NN of the stack learns a deterministic policy corresponding to investment decisions at the respective time-step, where gradients are passed through the entire NN to maximize the expected utility over the investment horizon. DH relies on the ability of gradient descent to identify a near-optimal solution. Through theoretical analysis and experiments, we show that its performance can falter in environments where the optimal action-value function is not governed by convexity assumptions, causing the algorithm to converge to suboptimal local minima. 

On the theoretical side, this paper provides a mathematical analysis of the limitations of DH, demonstrating its close relationship with traditional convex optimization. First, we show that in settings, where the utility function is concave and increasing, and transaction costs are convex, portfolio replication becomes a convex problem in the space of deterministic policies. This insight explains why DH performs well in such domains: under these conditions, a local, deterministic policy gradient method converges reliably to the global optimum. Notably, such environments are also well-suited to classical convex optimization techniques. Second, we examine scenarios where DH is applied to non-convex environments - such as those characterized by multimodal optimal action-value function, which may arise from non-convex transaction costs, liquidity constraints, market frictions, or regulatory restrictions. In such cases, we argue that DH fails with non-negligible probability. We provide concrete examples where these non-convexities play a central role, illustrating that non-convex behavior is not only possible but likely in realistic hedging scenarios. These theoretical results are supported by experiments across the described environments.

In our experiments, we propose an alternative RL-based approach inspired by the success of Monte Carlo Tree Search (MCTS) in solving complex games. We show that this method can consistently identify optimal courses of action in scenarios where DH fails. We also measure sample efficiency, where the reliance on gradient descent makes DH less sample-efficient in low-data regimes.
MCTS-based algorithms represent the state-of-the-art and most extensively studied method for a wide range of sequential decision-making problems - ranging from discrete, perfect information scenarios (such as combinatorial games and puzzles) at the one end, to continuous, imperfect information scenarios (as encountered in real-world decision-making) at the other. Typical applications include puzzles~\cite{Schadd} (in single-player settings), combinatorial games such as Chess and Go~\cite{silver2016mastering,schrittwieser2020mastering} (in two-player settings), as well as non-deterministic and imperfect information games, such as the Sailing domain~\cite{vanderbei1996,peret2004} and a variety of video games~\cite{guo2014deep,gvgaibook2019,antonoglou2022planning}. For comprehensive reviews, see~\cite{Browne2014} and~\cite{wiechowski2022}. %
Conversely, derivatives hedging can be conceptualized as a \emph{\lq\lq{}a game with the world\rq\rq{}}, where one player (the investor) bets on future outcomes, while the other player (the market) determines the actual outcomes~\cite{Szehr2023}. Shafer and Vovk~\cite{Shafer2001} argue that this interpretation is foundational and can serve an axiomatic role in the development of financial mathematics and probability theory as a whole. This game-theoretic interpretation aligns naturally with viewing replication as a sequential optimization problem of expected terminal utility: in each round the investor places a bet on the best composition of the replication portfolio; subsequently the market transitions to the next state until the asset matures and the games' outcome is revealed to the investor. Specifically, we investigate an AlphaZero/MuZero-based system for constructing replication portfolios under market incompleteness and compare them to DH. AlphaZero~\cite{Anthony2017,silver2017mastering} is an advanced variant of MCTS that integrates NNs to guide the construction of MCTS' decision tree and to evaluate leaf nodes, which demonstrated tremendous success in high-complexity, large state-space games. MuZero~\cite{schrittwieser2020mastering} is an extension of AlphaZero that requires no access to an environment simulator but learns an internal model to predict the environment's behavior purely from data. The following facts motivate our investigation of MCTS (and AlphaZero/MuZero in particular):
\emph{1.} MCTS solely relies on contractual cash flow information without requiring intermediate reward signals or external pricing. \emph{2.} NN-enhanced MCTS variants achieve state-of-the-art performance in complex games and general game playing. Interpreting utility optimization as a stochastic, multi-turn, two-player game, this suggests a strong potential for financial planning tasks. AlphaZero/MuZero's ability to handle large state and action spaces makes them well-suited for discretized market models and scaling to systems requiring intelligent multi-asset trading coordination. \emph{3.} One of the central contributions of this paper is to demonstrate that AlphaZero/MuZero can consistently identify optimal replication strategies even in settings where DH converges to local optima. Our experiments show that these methods achieve the same levels of test loss as DH while requiring fewer training samples. Since training DH on real market data is often impractical due to its substantial data requirements, improving our ability to extract effective policies from smaller datasets is an important direction for research.

To maintain focus and conciseness - and in view of page limitations - we leave several interesting research directions outside the scope of this paper.
(1) We do not compare our AlphaZero and MuZero agents with reinforcement-learning methods other than DH. In this work, the primary role of the AlphaZero and MuZero agents is to serve as reference baselines (i.e., ground truth) where a dynamic-programming approach is infeasible.
(2) Following the methodology of the original DH study~\cite{buehler2019deep}, we do not include experiments on historical data.

\section{Replication strategies}\label{se:replStrategies}
\subsection{Replicating portfolios} \label{se:replPortfolios}
We follow roughly the notation and exposition of the pioneering work on {Deep Hedging}~\cite{black1973pricing}.
We consider a discrete-time financial market with trading dates $0=t_0<t_1<...<t_n=T$.
Market information becomes available successively, where at each $t_k$ the investor first observes the new market state and then adjusts his positions for the investment period from $t_k$ to $t_{k+1}$.
At time $0$ the investor sells a complex asset $Z$ that matures at $T$. This generates an immediate cashflow\footnote{Focusing on the construction of replication portfolios rather than asset pricing, this article assumes that the price $p_0$ is given exogenously.
Pricing is discussed, e.g., in~\cite{buehler2019deep}.} of $p_0\geq0$ to the benefit of the investor but also a random liability $Z_n\leq0$ at $T$. The market is comprised of $\adim$ \lq\lq{}simple\rq\rq{} financial instruments\footnote{These instruments represent general drivers of market risk.
They could include cash, equity and commodity prices, assets that represent market volatility, etc.} $(X^i)_{i=1,...,\adim}$ with value $X_k^i$ at $t_k$.
To offset the risk from $Z_n=Z_n(X_n)$ the investor maintains a replication portfolio of the form $\Pi_k=\sum_{i=0}^{\adim} \delta_k^iX_k^i=\delta_k\cdot X_k,$ where $\delta_k^i$ denotes the holdings in the asset $X_k^i$ at time $t_k$ and $\delta_k=(\delta_k^0,...,\delta_k^\adim)$, $X_k=(X_k^0,...,X_k^\adim)$. Our index convention is that immediately after observing $X^i_k$ at $t_k$ the investor adjusts his position from $\delta_k^i$ to $\delta_{k+1}^i$. To reflect cash holdings without interest we set $X_k^0=1$. The overall positions $\delta_k$ are subject to constraints arising, e.g., from liquidity, asset limitations or overall trading restrictions. We assume, in particular, that trading is self-financed, i.e.~any adjustment of positions $\delta_k^1,...,\delta_k^\adim$ is reflected by a respective adjustment $\delta_k^0$ on the cash holdings such that $\delta_k\cdot X_k=\delta_{k+1}\cdot X_{k}$.
We assume that the transaction costs incurred at $t_k$ are of the form~$c_k=c_k(\Delta\delta_k,X_k)\leq0$ %
with $\Delta\delta_k=\delta_{k+1}-\delta_{k}$. The investor's terminal wealth at time $T$ is then
\begin{align}PL_{n}=p_0+Z_n+\delta_n\cdot X_n+\sum_{k=0}^{n-1} c_k(\Delta\delta_k,X_k).\label{PLEvolution}
\end{align}
Formally, the investor's objective is to choose a trading strategy (i.e.~policy, see below) to maximize\footnote{Note that this constitutes a constrained optimization problem over policies, where the sequence $(\delta_0, \ldots, \delta_{n-1})$ must satisfy all trading constraints such as the self-financing condition.} the expected utility
\begin{align}\max_{\delta_1,....,\delta_{n}}\mathbb{E}\left[u(PL_n)\right],\label{eq:targetProblem}
\end{align}
where the utility function $u$ quantifies the investor's risk preference. For exact replication one could minimize a loss $l$ between $Z_n$ and $\Pi_n$, corresponding to $l=-u$.

\subsection{Maintenance of the replication portfolio as a Markov Decision Process}\label{se:replPortfolioMaintan}
Markov decision processes (MDPs) constitute a mathematical framework that formalizes sequential decision-making problems, where an agent interacts with an uncertain environment over a sequence of time steps. Formally an MDP is a tuple $\mathcal{M}=(\mathcal{S},\mathcal{A},\tkernel,r,\mu,T)$,
where $\mathcal{S}$ is a set of admissible states, $\mathcal{A}$ are possible actions, $\tkernel\in[0,1]$ denotes the environment's transition kernel, i.e.~$\tkernel(s'|s,a)$ is the probability of entering state $s'\in\mathcal{S}$ given that action $a\in\mathcal{A}$ is taken in state $s\in\mathcal{S}$, and $r(s,a,s')$ is the reward generated through this transition. $\mu(s)$ is the initial distribution over states and $T$ is the MDP's planning horizon. The interactions take place at the times $0=t_0<t_1<...<t_n=T$. Occasionally we add a subscript, $t_k$ or simply $k$, to emphasize that a certain quantity belongs to a specific point in time. In what follows we assume that the time variable $t_k$ is always part of the state. Starting with $\mu(s)$, at each step of the MDP the agent first observes $s\in\mathcal{S}$ and then chooses $a\in\mathcal{A}$ according the policy $\pi(s,a)$. Subsequently the environment transitions to the next state $s'$ according to $\tkernel(s'|s,a)$ and a reward of $r(s,a,s')$ is granted to the agent. {The return from $t_k$ is $G_k=\sum_{k'=k}^{n-1}r(S_{k'},A_{k'},S_{k'+1})$.} Performance is measured by the state-value and action-value functions, which are defined for a given policy $\pi$ and all $s_k\in\mathcal{S}$, $a_k\in\mathcal{A}$ as
\begin{align*}
V^\pi(s_k)&=\mathbb{E}_\pi\left[G_k|S_k=s_k\right],\quad Q^\pi(s_k,a_k)=\mathbb{E}_\pi\left[G_k|S_k=s_k,A_k=a_k\right].
\end{align*}
The random variables $S$, $A$ represent the choice of state or action and the state-value or action-value functions represent the total expected reward when starting at state $s$, or, respectively, at the state action pair $(s,a)$, and acting according to $\pi$. Although we include the time variable in the state, it will sometimes be convenient to emphasize the specific point in time writing explicitly $V_{k}=V|_{\mathcal{S}_{k}}$ and $Q_{k}=Q|_{\mathcal{S}_{k}\times\mathcal{A}}$, where $\mathcal{S}_{k}\subset\mathcal{S}$ are the state at ${t_k}$. Finally, the objective function of an MDP with initial distribution $\mu$ and policy $\pi$ is defined as $J_{\mu}^{\pi} = \int_{\mathcal{S}} V^{\pi}(s) \textnormal{d}\mu(s).$ It is this objective that many RL algorithms optimize.

Portfolio replication can be phrased in the language of MDPs, where the agent represents the investor and the environment represents the market. We assume the absence of market impact (as in DH), i.e.~market components of the state evolve independently of the chosen actions.
This is a significant simplification for MDP theory and algorithm design, but it is an underlying premise behind the common framework of stochastic market models in the first place. For convenience we also assume Markovianity, which means that the distribution of future market states depends only on the current market, not on its history.\footnote{If the market is only partially observable, Markovianity can be ensured by incorporating the entire market observation history into the state.}
In this situation a possible choice of state representation at time $t_k$ could be
\begin{align}\label{eq:simpleState}S_k=(t_k,\delta_{k},X_k, p_0+\Pi_k+\sum^{k-1}_{j=0}c_j(\Delta\delta_j)),\end{align}
with $\Im(S_k)\subset\mathcal{S}$. The agent observes this state and chooses an action $A_k\in\mathcal{A}$ %
subject to all necessary trading constraints. As trading strategies are self-financed, we will interpret $A_k=(\delta_{k+1}^1,...,\delta_{k+1}^\adim)\in\mathcal{A}\subset{\mathbb{R}}^\adim$ and the cash holdings $\delta_k^0$ as part of the state. Maximizing the utility of terminal wealth corresponds to a unique reward\footnote{For assets with a term structure, rewards occur exclusively as a result of contractual cash flows.} at maturity
\begin{align*}
r(S_{k-1},A_{k-1},S_{k})=\begin{cases}0\ &\textnormal{if}\ k<n,\\
u(PL_n)=-l(PL_n)\ &\textnormal{else}.
\end{cases}
\end{align*}
In this formalism, Equ.~\eqref{eq:targetProblem} is equivalent to finding an optimal policy $\pi^*\in\argmax_\pi V^{\pi}$ {(provided that it exists, cf.~Sec.~\ref{se:backgroundAndSetting})}. While an optimal policy can be computed in $n$ steps using dynamic programming~\cite{sutton2018reinforcement} each step involves the expensive computation of expectations making straight-forward dynamic programming unfeasible even for moderate $T$. Numerous methods aim at solving this problem efficiently.
\subsection{Deep Hedging}\label{sec:DH}
DH approaches the sequential decision-making problem~\eqref{eq:targetProblem} using a stack of shallow NNs, or alternatively, a recurrent NN or LSTM.
The sequential structure of DH reflects the temporal structure of the problem~\eqref{eq:targetProblem}, where the $k$-th NN of the stack (or the $k$-th argument of the recurrent NN or LSTM) represents a \emph{deterministic} policy $F_k$ for taking action $A_k=(\delta_{k+1}^1,...,\delta_{k+1}^\adim)$ based on the state $S_k$ at time $t_k$.
Let $\mathcal{NN}_{\adim_0\adim_1}$ denote the set of all feed-forward NNs mapping from $\mathbb{R}^{\adim_0}\rightarrow\mathbb{R}^{\adim_1}$.
Although the specific architectural details may vary, a basic DH system chooses replication strategies from the set (see~\cite{buehler2019deep} for details):
$$\mathcal{NN}=\left\{(F_k)_{k=0,...,n-1}\ |\ (\delta_{k+1}^1,...,\delta_{k+1}^\adim)=F_k(S_k), F_k\in\mathcal{NN}_{\dim{(S_k)} \adim} \right\}.$$
The rationale behind this representation is that the action $(\delta_{k+1}^1,...,\delta_{k+1}^\adim)$ (of dimension $\adim=\dim{\mathcal{A}}$) computed by the NN $F_k\in\mathcal{NN}_{\dim{(s_k)} \adim}$ for the state $s_k$ (of dimension $\dim{(s_k)}>\adim$) becomes part of the subsequent state $s_{k+1}$ (of dimension $\dim{(s_{k+1})}>\adim$), cf.~\eqref{eq:simpleState}. Refer to~\cite[Figure 1]{buehler2019deep} for an illustration of the described DH architecture.

For training random $\mathbb{R}^\adim$-valued market evolutions $X_k$, $k=0,...,n$ are sampled from a stochastic market model/market simulator. The entire sequence $F=(F_0,F_1,...F_{n-1})\in\mathcal{NN}$ is then optimized end-to-end via gradient descent. Thus DH can be interpreted either as a supervised learning process, with samples
$\left\{(x^{(i)}_0,...,x^{(i)}_n),u(PL_T(x^{(i)}_0,...,x^{(i)}_n,F))\right\}_{i}
$ maximizing $\mathbb{E}[u(PL_T(X_0,...,X_n,F))]$ or as an RL process with deterministic policy and episodic reward $u(PL_n(X_0,...,X_n,F))$. For simple market models (such as discretized geom.~Brownian motion, Heston Model, or simple yield curve models) the number of sampled market evolutions for training typically lies in the thousands, see e.g.~\cite{Teich00} and below. Although DH might be seen as {model-free} in that its NN structure does not make explicit use of a market model, such numbers can realistically be obtained only from a market simulator (rather than the real markets). Effectively, this brings DH close to {model-based} RL (such as MCTS, see below). Contrary to algorithms like value iteration, DH only searches through part of the state space, optimizing the objective $J_\mu^\pi$ only for the specific initial conditions outlined by the simulated market $X_k$.

\subsection{From financial games to MCTS, AlphaZero and MuZero} \label{se:financialGamesAndTraining}
Inspired by the success of MCTS-based algorithms across a wide range of games, we frame the management of a replication portfolio as a multi-turn, two-player game - where the investor places bets and the market determines outcomes - and apply MCTS variants to solve it. Unlike \lq\lq{}brute-force Monte Carlo methods\rq\rq{} that sample environment transitions and actions indiscriminately, MCTS performs a \lq\lq{}guided Monte Carlo search\rq\rq{}, selectively constructing a restricted search tree based on reward signals~\cite{Chang2005,Kocsis2006,Kocsis2006_}.
MCTS is model-based, as it samples transitions from $\tkernel$. Its action selection balances exploration (adding new nodes) and exploitation (choosing high-reward paths), guided by reward statistics at each node. This balance is effectively achieved using bandit-based methods like the UCT algorithm~\cite{Kocsis2006,Kocsis2006_}.

UCT is guaranteed to identify optimal actions in stochastic environments in the limit of infinite time and computational resources~\cite{Kocsis2006,Kocsis2006_,maxEntroPlanning}. In principle, this allows UCT to be used for constructing replication portfolios with a guarantee of performance, in contrast to DH that may converge to a poor local optimum. However, in games with large state spaces or branching factors it becomes unfeasible to sample all nodes of the tree sufficiently often. AlphaZero combines UCT with NNs to guide the construction of the planning tree and to evaluate leaf-nodes~\cite{Anthony2017,silver2017mastering}. These NNs provide estimates of rewards and visitation statistics for nodes not yet explored by Monte Carlo simulations, reducing the need for the costly simulations. This improves computational efficiency, but the estimates might be less accurate and the NNs themselves may converge to poor local optima - losing UCT's performance guarantees. Stochastic market models often have continuous state spaces, which can be discretized by placing variables on a grid. Smaller grid sizes reduce discretization error but increase the size of the state space. This makes AlphaZero variants well-suited for replication problems in discretized markets. Discrete action spaces, in turn, align naturally with practical trading constraints such as minimum order sizes imposed by brokers. Refer to App.~\ref{app:describeAlphaZero} and~\cite{Anthony2017,silver2017mastering} for a description of AlphaZero's architecture.

MuZero~\cite{schrittwieser2020mastering,antonoglou2022planning} extends AlphaZero by planning without prior knowledge of the model $\tkernel$, using a learned surrogate model instead of relying on $\tkernel$ to simulate trajectories.
MuZero is trained solely on observed data collected by playing games from beginning to end, generalizing the path-based training process of DH as described in Sec.~\ref{sec:DH}. To ensure a fair comparison when measuring sample efficiency, we assume a fixed dataset of market paths and compare MuZero (rather than AlphaZero) against DH. The absence of market impact yields an important simplification for our system design, see App.~\ref{app:describeMuZero} for details.

For complex games like chess or Go, AlphaZero has reached super-human level and it has beaten the world-champion chess software Stockfish\footnote{There is debate whether this benchmarking was appropriate and whether more advanced Stockfish variants perform comparably to AlphaZero~\cite{chess,NNChess}, but this is not the topic of our investigation.} in~\cite{schrittwieser2019mastering}, see also~\cite{schrittwieser2020mastering}. Such performance would be difficult to achieve with plain NN architectures, like DH. In what follows we analyze the convergence of DH both from a theoretical, cf.~Sec.~\ref{sec:algos}, and experimental, cf.~Sec.~\ref{sec:experiments}, perspective. While Sec.~\ref{sec:algos} is technical, both sections are self-contained and can be read independently.
\section{On the Optimality of Deterministic Policy Search for Replication Strategies}\label{sec:algos}
DH employs gradient methods to optimize deterministic policy NNs. The original work asserts that in this framework the optimal hedging strategy can be approximated arbitrarily well as the size of the NN increases. While this theoretical expressiveness is encouraging, the original work also acknowledges key limitations: \emph{\lq\lq{}there is no general result guaranteeing convergence to the global minimum in a reasonable amount of time\rq\rq{}}, and DH may escape local minima largely due to the noise inherent in stochastic gradient variants (SGD or ADAM). In this section, we provide a mathematical analysis of when deterministic policy gradient systems succeed - and when they fail - to identify globally optimal solutions. First, we show that in settings where utility is concave and increasing and transaction costs are convex, portfolio replication becomes a convex problem in the space of policies. This formalizes why DH performs well in such domains: we show that a local, deterministic policy gradient method is sufficient for reaching the global optimum. It is worth noting that these environments are also amenable to traditional convex optimization techniques. Second, when DH is applied to non-convex environments - such as those with multimodal optimal action-value functions (arising from non-convex costs, market frictions, regulatory constraints, etc.) - we argue that it fails with non-negligible probability. We characterize these failure modes and support our findings with empirical results in the next section. We phrase our result in the language of general MDPs and deterministic policy methods and comment on the application to DH.
\subsection{Mathematical Background and Setting}\label{se:backgroundAndSetting}
Any policy that maximizes action-values over all policies is optimal. For finite action spaces optimal policies and the optimal action-value function are defined via the relations~\cite{puterman2014markov}
\begin{align*}
    \pi^*\in\argmax_\pi\left\{Q^\pi\right\}\quad\textnormal{and}\quad Q^*=\max_\pi\left\{Q^\pi\right\}.
\end{align*}
But to employ gradient descent on deterministic policies, a continuous action space is required, which entails continuous state components. To formalize the search for optimal policies in this setting, the set of admissible policies has to be chosen carefully. Let $\Pi^c$ denote the set of continuous, deterministic policies and let $\pi\in\Pi^c$. Throughout we assume that $\mathcal{A}$ is convex and compact, the reward function $r$ is continuous and that the transition kernel $\tkernel$ is weakly continuous, i.e.~if $f\in C(\mathcal{S})$ then the map $(s,a)\mapsto\int_{\mathcal{S}}f(s)\textnormal{d}\tkernel(s'|s,a)$is continuous.

The value functions satisfies the following backwards recurrence relations: By convention $V_n=V^*_n=V^{\pi}_n=0$ and for $k<n$ it holds that
\begin{align}
  Q^*_{k-1}(s,a)&=\int_{\mathcal{S}}[r(s,a,s')+V^*_{k}(s')]d\tkernel(s'|s,a),\ V^*_{k-1}(s)=\max_{a\in\mathcal{A}} Q^*_{k-1}(s,a),\label{eq:optbackwardsRecursion}\\
   Q^{\pi}_{k-1}(s,a)&=\int_{\mathcal{S}}[r(s,a,s')+V^{\pi}_{k}(s')]d\tkernel(s'|s,a),\ V^{\pi}_{k-1}(s)=Q^{\pi}_{k-1}(s,\pi(s))\label{eq:pibackwardsRecursion}.
\end{align}
Our assumptions (on $\mathcal{A},r,\tkernel$, $\pi$) are chosen such that all quantities in~\eqref{eq:optbackwardsRecursion} and~\eqref{eq:pibackwardsRecursion} remain well-defined and are continuous throughout the recursion. However, there are no continuous optimal policies for many RL problems\footnote{{The standard framework to address this is that of universally measurable policies introduced in \cite{bertsekas1996stochastic}. To keep our discussion accessible we chose to restrict the class of RL problems rather than delving into measure theory.}} - in fact~it is straight-forward to find a continuous $Q$ such that any policy $\pi$ with $\pi(s)\in \argmax Q(s,\cdot)$ must be discontinuous. For instance, if $\mathcal{A}\subset \mathbb{R}$ then $\pi^*(s) = \inf \argmax Q(s,\cdot)$ is only lower semi-continuous,~cf. Lem.~\ref{le:infargmaxLSC} in App.~\ref{ap:measurableoptpi} and \cite{pollard2001user}. This can lead to gaps $\inf_{\pi\in\Pi^c} \|V^*-V^{\pi}\|_{\infty} > 0$ and $\inf_{\pi\in\Pi^c} \|Q^*-Q^{\pi}\|_{\infty} > 0$. On the other hand the NNs with continuous activations inherent to DH are continuous functions. As we are mostly interested in the analysis of this framework, we will assume the policy set $\Pi^c$ and restrict ourselves to domains for which $\inf_{\pi\in\Pi^c} \|Q^*-Q^{\pi}\|_{\infty} =0$, leaving the discussion of emerging gaps to App.~\ref{ap:measurableoptpi}.

\subsection{Portfolio Replication and Convex Optimization}
\label{sec:portfolioReplAndConvexOpt}
The following theorem reveals that portfolio replication is related to convex optimization over $Q^*$.
\begin{restatable}[]{theorem}{convex}\label{thm:convex}
Consider the portfolio replication MDP described in Sections~\ref{se:replPortfolios},~\ref{se:replPortfolioMaintan} and the setting of Sec.~\ref{se:backgroundAndSetting}. Assume a concave and increasing utility $u$ and a convex cost function $c$. Then the optimal action value function $Q^*$ is concave in the action argument.
\end{restatable}
As a consequence every local maximum of $Q^*(s,\cdot)$ is global (and all global maxima form a convex set)~\cite[Section 2.5]{convexOptim}, i.e.~$Q^*(s,\cdot)$ is a unimodal function of the action argument. A function that is not unimodal is called multimodal, refer to~\cite{convexOptim} for details. In App.~\ref{ap:couterexamples} we provide counterexamples to illustrate the necessity of the assumptions in Theorem~\ref{thm:convex}, thereby demonstrating that $Q^{*}$ can be {non-concave} in several practically relevant scenarios: We show that when employing a quadratic (i.e.~non-increasing) utility, $Q^{*}$ is no longer unimodal. Similarly we show that non-convex transaction costs lead to multimodal $Q^{*}(s,\cdot)$ for some $s$. Imposing cash constraints (i.e.~restricting $\delta^0_k$), the domain of $Q^*(s,\cdot)$ might become disconnected for some $s$. The proof Thm.~\ref{thm:convex} makes use of two elementary lemmas. We show the lemmas here but give their proofs in App.~\ref{app:convexLemmas}.

\begin{restatable}[]{lemma}{concavecomp}
\label{le:concavecomp}\textnormal{(Concavity-preserving composition)} Let $f:\mathbb{R}^n\rightarrow \mathbb{R}$
and $g:\mathbb{R}^m\rightarrow \mathbb{R}$, $n,m \geq 1$, be concave functions that are
increasing in the first argument. Then $f\circ (g\times \mathrm{id}_{\mathbb{R}^{m-1}})$ is a concave function and
increasing in the first argument.
\end{restatable}

\begin{restatable}[]{lemma}{concavemax}
\label{le:concavemax}\textnormal{(Maximum preserves concavity)}
Let $f:\mathbb{R}^n\times \mathbb{R}^m \rightarrow \mathbb{R}$, $n,m \geq 1$, be a concave function. Then $g:\mathbb{R}^n \rightarrow \mathbb{R},\ x\mapsto g(x)= \max_{y} f(x,y)$ {(assuming the right hand side exists)} is a concave function. Moreover if $x_1 \mapsto f(x_1,x_2,y)$ is increasing, then $x_1 \mapsto g(x_1,x_2)$ is increasing, too.
\end{restatable}

\begin{proof}[Proof of Thm.~\ref{thm:convex}]
Let $S_k = (k,\delta_k,X_k,W_k)$ with $W_k=p_0+\Pi_k+\sum^{k-1}_{j=0}c_j(\Delta\delta_j,{X_j})$ and $A_k = \delta_k^{1:\adim}$. We introduce the function $h$ that describes the update rule of $\delta_k^0$ due to the self-financing constraint
$$
\delta_{k+1}^0 = \delta_{k}^0 - (\delta_{k+1}^{1:\adim}-\delta_{k}^{1:\adim})X_k^{1:\adim} =
h(\delta_k,\delta_{k+1}^{1:\adim},X_k)
$$
and a function $g_k$ that describes the updating rule of $W_k$, 
$$
W_{k+1} = W_k - \delta_k X_k + \delta_{k+1} X_{k+1} - c_k(\delta_{k+1}^{1:\adim}-\delta_{k}^{1:\adim},X_k) = g_k (W_k,\delta_k,\delta_{k+1},X_k,X_{k+1}).
$$
Let $x_k,w_k,\deval_k$ denote realizations of the random variable $X_k,W_k,\delta_k$. To shorten the notation\footnote{With this shift from $V_n=0$ we can omit rewards in our proofs.} we set $V_n^* = r_n$. The proof proceeds inductively by noting that the reward at maturity $r_n = u(w_n + z_n(x_n))$ is a concave and increasing function in $w_n$ (for fixed $x_n$) and repeating for $k<n$ the following steps:

    \emph{(i)}
    Substitute $w_{k+1} = g_k (w_k,\deval_k,\deval_{k+1},x_k,x_{k+1})$ in $V_{k+1}^*$ and $\deval_{k+1}^0 = h(\deval_k,\deval_{k+1}^{1:\adim},x_k)$ for the first component of $\deval_{k+1}$. Fix $x_k, x_{k+1}$. {Since} $V_{k+1}^*$, $g_k$ and $h$ are concave and increasing functions in $w_{k+1}$, $w_k$ and $\deval_k^0$ respectively, then, applying Lem.~\ref{le:concavecomp} twice, the result $V_{k+1}^*\circ (g_k \times \mathrm{id}) \circ (h \times \mathrm{id})$ is a concave function that is increasing in $w_k$ and $\deval_k^{0}$. 
      
    \emph{(ii)} Fix $x_k, x_{k+1}$. Marginalization over $x_{k+1} \sim X_{k+1} | (X_k=x_k)$ (a convex combination of concave functions that are increasing in $w_k$ and $\deval^0_k$) yields the action-value function
    \begin{align*}
    &Q_k^*(\deval_k,x_k,w_k,\deval^{1:\adim}_{k+1})\\
    &= \ev_{x_{k+1}\sim X_{k+1}| X_k=x_k}[V_{k+1}^*(g_k(w_k,\deval_k,h(\deval_{k},\deval_{k+1}^{1:\adim},x_k),\deval_{k+1}^{1:\adim},x_k,x_{k+1}) ],
    \end{align*}
    which is concave and increasing in $w_k$ and $\deval^0_k$. The assumptions of Sec.~\ref{se:backgroundAndSetting} guarantee that this expectation is finite. %
    
    \emph{(iii)} The following state-value function is well-defined as the concavity of $Q^*(s,\cdot)$ implies continuity and $\mathcal{A}$ is compact,
    $$
    V_{k-1}^*(\deval_k,x_k,w_k) = \max_{\deval^{1:\adim}_{k+1}} Q_k^*(\deval_k,x_k,w_k,\deval^{1:\adim}_{k+1}).
    $$
    By Lem.~\ref{le:concavemax} this is a concave function (for fixed $x_k$) that is increasing in $w_k$ and $\deval^0_k$.
\end{proof}
The following theorem establishes the relation between Thm.~\ref{thm:convex} and gradient methods that search $\Pi^c$ for {deterministic} policy maxima of $\pi\mapsto J_{\mu}^{\pi}$. In a general MDP context the theorem ascertains {a form of equivalence} of the unimodality of action-value and objective functions.
\begin{restatable}[Unimodality of Action-Value and Objective Functions]{theorem}{unimodal}
\label{thm:unimodal}
The following assertions hold for an MDP $\mathcal{M} = (\mathcal{S},\mathcal{A},\lambda,r,\mu, T)$ as in Sec.~\ref{se:backgroundAndSetting}:
\begin{enumerate}
\item If for for all states $s \in \mathcal{S}$ the map $a\mapsto Q^*(s,a)$ is unimodal, then the map $\pi \mapsto J_{\mu}^{\pi}$ defined on $\Pi^c$ is unimodal.
\item Assume that $\mathcal{A}\subset\mathbb{R}$. If there exists a state $s \in \mathcal{S}$ such that $a\mapsto Q^*(s,a)$ is multimodal then there exists an initial distribution $\mu'$ such that $\pi \mapsto J_{\mu'}^{\pi}$ is multimodal on $\Pi^c$.
\end{enumerate}
\end{restatable}
\begin{proof}[Proof of Thm.~\ref{thm:unimodal}]
\emph{1.}
Assume an optimal policy $\pi^*\in\Pi^c$, cf.~Sec.~\ref{se:backgroundAndSetting} and App.~\ref{ap:equivalence} for details, and let $\pi_0 \in \Pi^c$ a non-optimal policy. Let $\mu'$ be an initial distribution.
We aim to show that there exists a continuous curve $[0,1]\rightarrow\Pi^c, \alpha\mapsto\phi(\alpha)$ from
$\pi_0=\phi(0)$ to $\pi^*=\phi(1)$ such that $\alpha \mapsto J_{\mu'}^{\phi(\alpha)}$ is non-decreasing. Consequently, $\pi_0$ is not a local maximum. Let, as before, $k<n$ denote the current time index and $h = n-k$ the remaining horizon. Let $\pi_{0,k}$ and $\pi^*_k$ stand for the policies $\pi_0,\pi^*$ at time $k$ and for $0\leq h\leq n$ define policies $\pi_{h}\in\Pi^c$ at time $k$ as
$$
\pi_{h,k} = \begin{cases}
\pi_{0,k} &\text{if}\; k < n-h, \\
\pi_{k}^{*} &\text{otherwise}.
\end{cases}
$$
Note that $\pi_h$ equals for $h=0$ to the non-optimal policy $\pi_0$ and that $\pi_n = \pi^*$ is the optimal policy. We will construct curves $\phi_{h}:[0,1]\rightarrow\Pi^c$ such that $\alpha \mapsto J_{\mu'}^{\phi(\alpha)}$ is non-decreasing by induction over the remaining horizon $1\leq h\leq n$. For fixed $h$ and any $s\in \mathcal{S}_{n-h}$ we have that $Q^{\pi_{h-1}}(s,\cdot) = Q^{\pi_{h}}(s,\cdot) = Q^{*}(s,\cdot)$ is unimodal by assumption and the action $\pi_{h}(s)$ is a global maximum. Thus there exists a curve $\phi_{h}:[0,1]\rightarrow\Pi^c$ such that $\phi_h(0) = \pi_{h-1}$, $\phi_h(1) = \pi_{h}$ and $\alpha \mapsto Q^*(s,\phi(\alpha)) = Q^{\phi(\alpha)}(s,\phi(\alpha)) = V^{\phi(\alpha)}(s)$ is is non-decreasing. Denote by $\nu_{n-h}$ the visitation measure at horizon $h$ (time $n-h$) arising from $\pi_0$ and $\mu'$. Since $s\in \mathcal{S}_{n-h}$ could be chosen arbitrary, the map $\alpha \mapsto J_{\mu'}^{\phi(\alpha)} = \int_{\mathcal{S}}V_{n-h}^{\phi_{h}(\alpha)}(s)d\nu_{n-h}(s)$ is non-decreasing (due to the monotony of the Lebesgue integral). The curve $\phi$ is the concatenation of 
curves $\phi=\oplus_{i=1}^n \phi_i$.

\emph{2.} Fix $s$ as in the theorem. Wlog.~$Q^*(s,\cdot)$ has two non-degenerate local maxima. Then there exist three actions $a_0 < a_2 < a_1$ such
that $\min \{Q^*(s,a_0), Q^*(s,a_1)\} > Q^*(s,a_2)$. There might be no optimal $\pi^*\in\Pi^c$, so we choose $\pi_0\in\Pi^c$ with $\pi_0(s) = a_0$ and that is close enough to $\pi^*$, i.e.~$Q^{\pi_0}(s,a_0) > Q^{*}(s,a_0)-\epsilon$ for some $\epsilon>0$ (e.g.~$\epsilon < (\min \{Q^*(s,a_0), Q^*(s,a_1)\}- Q^*(s,a_2))/3$). We choose $\pi_1$ in a similar fashion. Let $\phi:[0,1] \rightarrow \Pi^c$ be any continuous curve of policies connecting $\pi_0$ and $\pi_1$. By the intermediate value theorem there exists $\alpha_2 \in (0,1)$ such that $\phi(\alpha_2)(s) = a_2$. This means that $Q^{\phi(\alpha_2)}(s,a_2) \leq Q^*(s,a_2)$. Let $\mu' = \delta_s$ be the Dirac measure at $s$. Then $J_{\mu'}^{\phi(0)} = V^{\pi_0}(s) = Q^{\pi_0}(s,a_0)$, $J_{\mu'}^{\phi(1)} = V^{\pi_1}(s) = Q^{\pi_1}(s,a_1)$
and $J_{\mu'}^{\phi(\alpha_2)} = Q^{\phi(\alpha_2)}(s,a_2) \leq  Q^*(s,a_2)$. As this reasoning remains valid for any continuous $\phi$ from $\pi_0$ to $\pi_1$ there exists $\alpha_2(\phi)$ so that
$$
\min \{J_{\mu'}^{\pi_0}, J_{\mu'}^{\pi_1}\}- J_{\mu'}^{\phi(\alpha_2(\phi))} > 
\min \{Q^*(s,a_0), Q^*(s,a_1)\}- Q^*(s,a_2)-\epsilon,
$$
which witnesses the multimodality of $\pi\mapsto J_{\mu'}^\pi$.
\end{proof}

Taken together Thm.~\ref{thm:convex} and Thm.~\ref{thm:unimodal} imply that under the assumptions of Thm.~\ref{thm:convex} deterministic policy search optimizes an unimodal objective function. In this case we can expect gradient methods to identify the optimal policy. However, if the assumptions are violated, our examples show that $J_\mu^\pi$ might become multimodal or the domain of $Q^*$ might become disconnected, posing potential problems for gradient methods. Stochastic gradient methods can escape local optima due to randomness, but this can also lead from global to local optima. Heuristic approaches like simulated annealing aim for global optima, though their success depends on hyperparameters and the optimization landscape, making them domain-specific in general. These challenges of deterministic policy search algorithms highlight the importance of noise in exploration - the described limitations do not apply to stochastic policy-based methods like AlphaZero or off-policy methods like Deterministic Policy Gradient (DPG)\footnote{Consistent with our findings, \cite{silver2014deterministic}~noted that on-policy DPG suffers from insufficient exploration.}. Finally, as typical NN architectures are designed to approximate continuous functions, both stochastic and deterministic methods that use NNs to represent policies suffer from a lack of expressiveness when optimal policies are not continuous.

\section{Experimental Findings}\label{sec:experiments}
To support the theoretical considerations of the previous chapters, we conduct experiments with DH and AlphaZero-based architectures in scenarios with {non-concave} ${Q}$-functions. In Sec.~\ref{sec:minimalEnvs}, we study simple archetypes of planning problems in stochastic environments. Sec.~\ref{sec:toyMarketEnvs} contains experiments with prototypical replication problems in simulated markets. Sec.~\ref{sec:sampleEfficiency} concludes the comparison by measuring the consumption of market data within DH and MuZero training processes.

\subsection{Deep hedging \& AlphaZero in minimal environments with multi-modal rewards}\label{sec:minimalEnvs}

\subsubsection{Learning a deterministic sequence of actions}
\label{se:learningDeterministic}
Consider an agent learning a predefined sequence of five consecutive actions $[-0.5,0.5,-0.5,0.5,-0.5]$ from an action space $\mathcal{A}=\{-1.0,-0.9,...,1.0\}$  ($n=4,\adim=1$), with states represented as in Equ.~\eqref{eq:simpleState}. Since $t_k$ is part of the state $s_k$, this corresponds to a simple assignment $t_k\mapsto a_k$, $k=0,...,4$. Rewards are computed using the bi-modal reward function $r$ in Fig.~\ref{fig:reward_plot}. For $k=0,2,4$ the reward is $r_k(x)=r(x)$ and for $k=1,3$ it is $r_k(x)=r(-x)$. In other words the positioning of the larger mode is such that it corresponds to the predefined sequence of optimal actions. For training only the accumulated reward is granted to the agent after each episode. Fig.~\ref{fig:histogram_plot:one} shows a histogram of the number of correct actions taken over $100$ independent training cycles (with Gaussian NN initialization) for both (domain-adapted) DH and AlphaZero. DH identifies the optimal sequence of actions in $\sim10\%$ of cases, as compared to AlphaZero, which identifies the sequence in $\sim85\%$ of cases. DH was subject to extensive hyperparameter tuning, whereas AlphaZero was used with a standard, untuned architecture. For details on the environment and training, see App.~\ref{app:learningDeterministic}. This environment reveals how gradient descent can converge to local optima in sequential decision tasks. If DH initially explores actions near $0.5$, it may settle there and miss the higher reward at $-0.5$ - a bias that persists throughout the decision sequence.
\begin{figure}[h]
    \begin{subfigure}[b]{0.33\linewidth}%
        \includegraphics[width=1.0\linewidth, trim=4mm 0mm 2mm 0mm, clip=True]{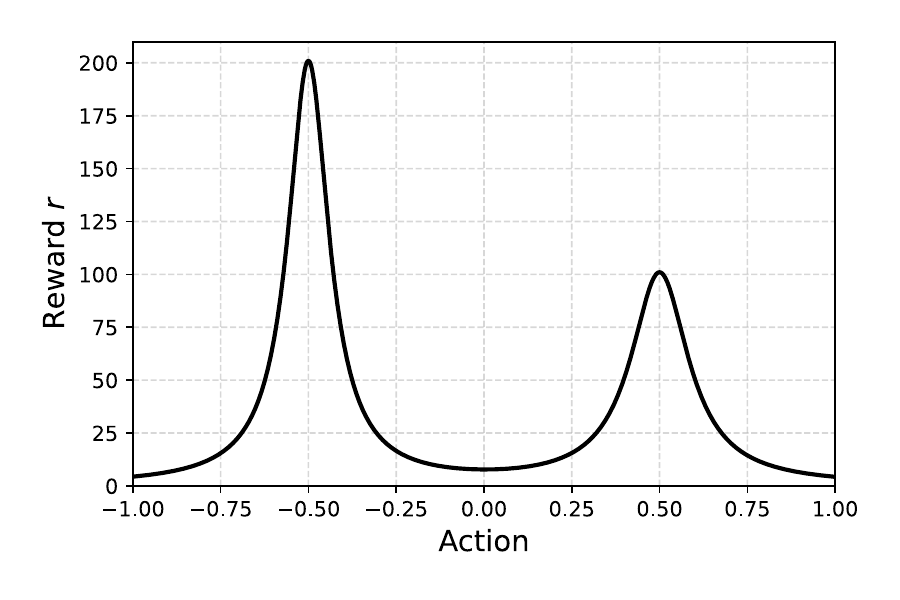}%
        \caption{Reward function.}%
        \label{fig:reward_plot}%
    \end{subfigure}%
    \begin{subfigure}[b]{0.33\linewidth}%
        \includegraphics[width=1.0\linewidth, trim=4mm 0mm 2mm 0mm, clip=True]{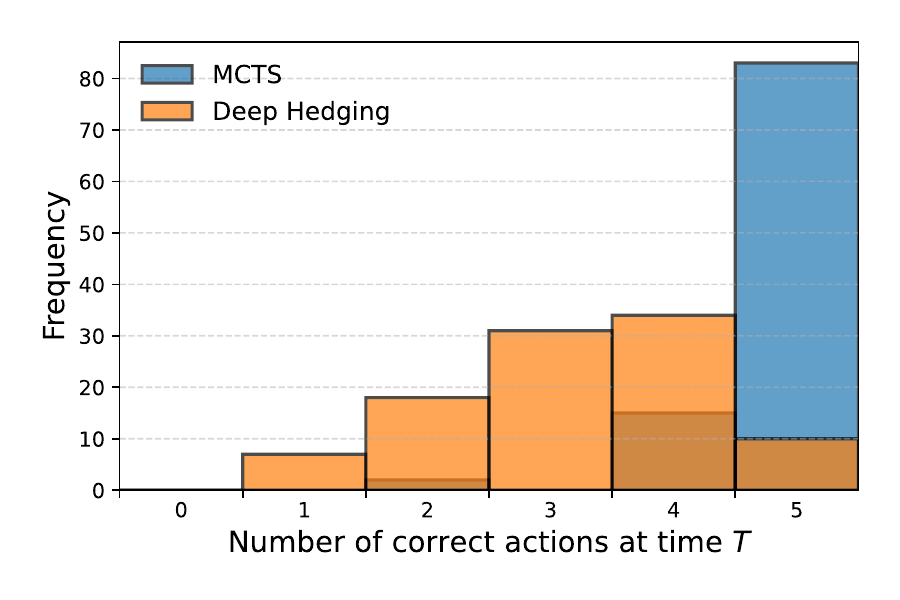}%
        \caption{Hist.~of correct actions.}%
        \label{fig:histogram_plot:one}%
    \end{subfigure}%
    \begin{subfigure}[b]{0.33\linewidth}%
        \includegraphics[width=1.0\linewidth, trim=4mm 0mm 2mm 0mm, clip=True]{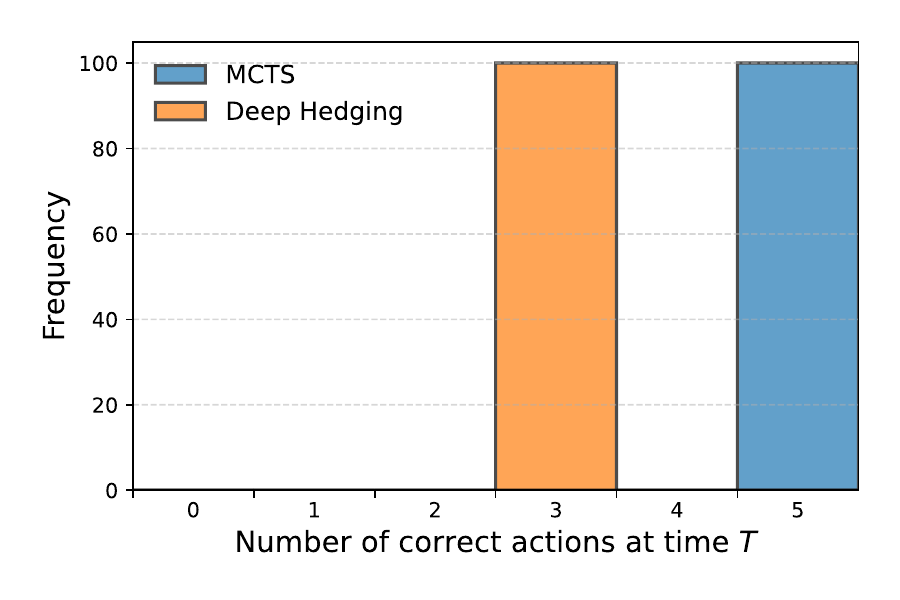}%
        \caption{Hist.~of correct actions.}%
        \label{fig:histogram_plot:two}%
    \end{subfigure}\\
    \caption{Learning deterministic assignments with bi-modal reward signals.}
    \label{fig:combined_plots}
\end{figure}
\subsubsection{Learning a deterministic portfolio composition}\label{se:learningKnown}
As a further illustration, we modify the setup of Sec.~\ref{se:learningDeterministic} to mimic the learning of a simple function that maps market state to optimal action (analogous to learning $\delta$ in the Black-Scholes framework). We choose a uniformly random market $X_k\in\{-0.5,0.5\}$ ($\adim=1, k=0,...,4$). If $X_k=-0.5$ the reward $r_k(x)=r(x)$ is chosen (as in Fig.~\ref{fig:reward_plot}) and if $X_k=0.5$ then $r_k(x)=r(-x)$. The optimal sequence of actions is then $a_k=X_k$. Fig.~\ref{fig:histogram_plot:two} shows a histogram of correct actions taken over $100$ independent training cycles (Gaussian NN initialization) for both DH and AlphaZero. Despite hyperparameter-tuning DH identifies the optimal sequence of actions in $0\%$ of cases, as compared to AlphaZero, which identifies the sequence in $100\%$ of cases, see App.~\ref{app:learningKnown} for experimental details.
\subsection{Deep hedging \& AlphaZero in toy examples of market models with non-concave $Q^*$}\label{sec:toyMarketEnvs}
We consider a prototypical setup for portfolio replication. At time $t=0$ an investor sells a European call option at price $p_0$, with a future payoff of $Z_T=-\max\{X_T-K,0\}$ (with $\adim=1$). The replication portfolio is dynamically adjusted to offset $Z_T$ by allocating a number $\delta_k$ of shares $X_k$ according to the portfolio equation~\eqref{PLEvolution}. We seek an exact match for $Z_T$, which is achieved by minimizing the mean-squared loss in~\eqref{eq:targetProblem}.

\subsubsection{Learning portfolio replication in a trinomial market with non-convex costs}\label{sec:toyMarketEnvs:trinomialAbsorbing}
We construct a replication problem with a {non-concave} $Q^*$ that is simple enough to be solved explicitly by dynamic programming. We set $X_0=K=5$, $p_0=0.4$ and we represent states as in~\eqref{eq:simpleState}, where we assume that $X_k\in\{1,2,3,...,9\}$. The agent makes five ($n=4$) consecutive portfolio adjustments choosing actions $\delta_k\in\mathcal{A}=\left\{0/20,1/20,...,19/20\right\}$. We assume that the stochastic process $X_0,X_1,...,X_4$ is Markovian with transition probability matrix given in Fig.~\ref{fig:trinomialMarket:c}. Transaction costs are of the form $c_k=\min\{0.25\cdot|{\Delta\delta_k}|,0.05\}$. Fig.~\ref{fig:trinomialMarket:a} presents a plot of values of $Q^*$ (at the initial state) at $t_0$ computed exactly using dynamic programming. Fig.~\ref{fig:trinomialMarket:b} shows histograms of the action choices at $t_0$ of trained DH and AlphaZero agents obtained from $100$ independent training cycles. DH identifies the correct mode ($12/20$) in $\sim26\%$ of cases, as compared to AlphaZero, which identifies the correct mode in $100\%$ of cases. For DH the observed frequencies of mode-choice reflect the probability distribution of selecting initial conditions (for stochastic gradient descent) when the NN is initialized with random Gaussian weights. This suggests that DH identifies the optimal mode when its NN has been initialized with parameters that lead to the correct optimal selection at the onset.
\begin{figure}[h]
    \centering
    \begin{subfigure}[b]{0.5\linewidth}%
        \centering%
        \includegraphics[width=1.0\linewidth]{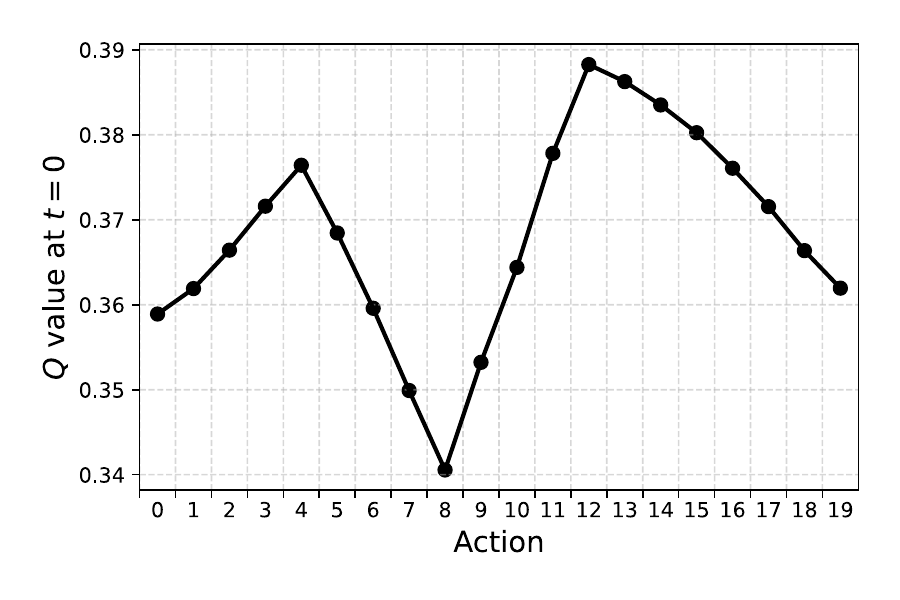}%
        \caption{$Q^*$-function at $t_0$.}%
        \label{fig:trinomialMarket:a}%
    \end{subfigure}%
    \begin{subfigure}[b]{0.5\linewidth}%
        \centering%
        \includegraphics[width=1.0\linewidth]{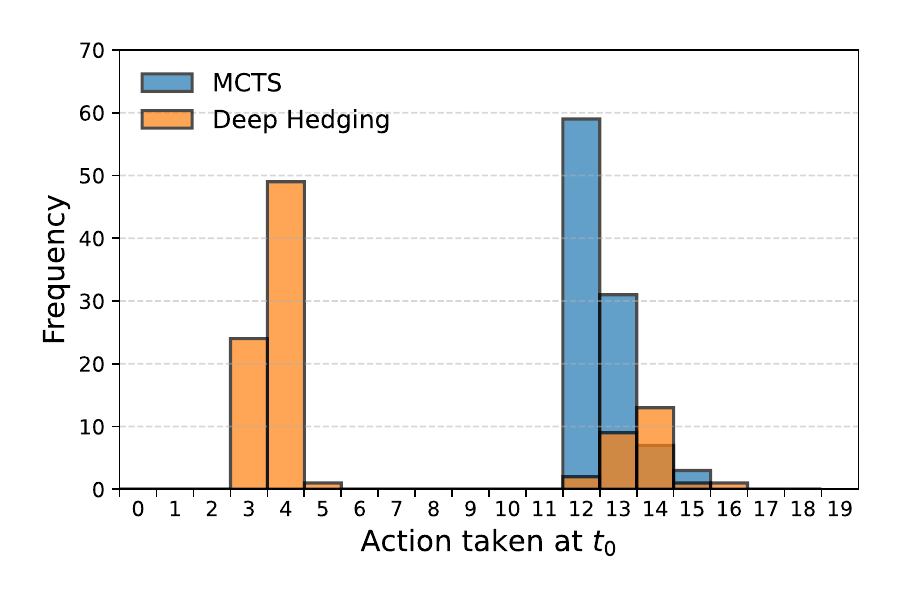}%
        \caption{Hist.~of actions chosen at $t_0$.}%
        \label{fig:trinomialMarket:b}%
    \end{subfigure}\\[2mm]
    \begin{subfigure}[b]{\linewidth}
    \begin{center}%
        $$%
            \begin{pmatrix}
            0.8 & 0.2 & 0 & 0 &\cdots & 0 & 0\\
            0.2 & 0.6 & 0.2 & 0 &\cdots  & 0 & 0\\
            0 & 0.2 & 0.6 & 0.2 &\cdots  & 0 & 0\\
            0 & 0 & 0.2 & 0.6 &\cdots  & 0 & 0\\
            \vdots & \vdots & \vdots & \vdots & \ddots & \vdots &\vdots\\
            0 & 0 & 0 & 0 & \cdots & 0.6 & 0.2\\
            0 & 0 & 0 & 0 &\cdots & 0.2 & 0.8\\
            \end{pmatrix}
        $$%
    \end{center}
    \vspace{0.3cm}
        \caption{Transition probabilities}
        \label{fig:trinomialMarket:c}
    \end{subfigure}
    \caption{Learning optimal actions in a market with bi-modal $Q^*$-function at init.~state.}
    \label{fig:trinomialMarket}
\end{figure}

\noindent
\begin{minipage}[t]{0.49\linewidth}%
    \centering%
    \includegraphics[width=\linewidth]{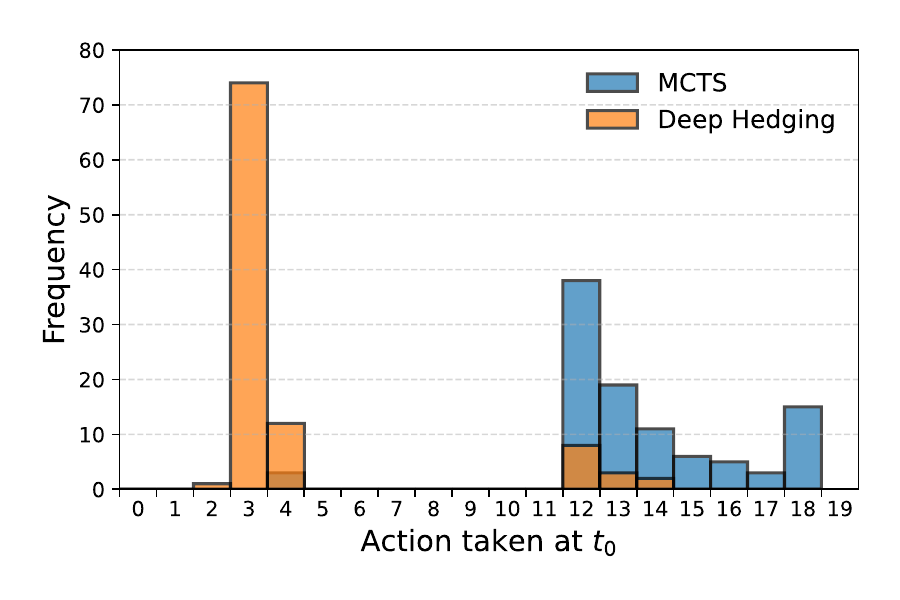}%
    \captionof{figure}{Hist.~of actions chosen at $t_0$.}%
    \label{fig:gbmMarket}%
\end{minipage}%
\hspace{0.02\linewidth}
\begin{minipage}[t]{0.49\linewidth}%
    \centering%
    \includegraphics[width=\linewidth]{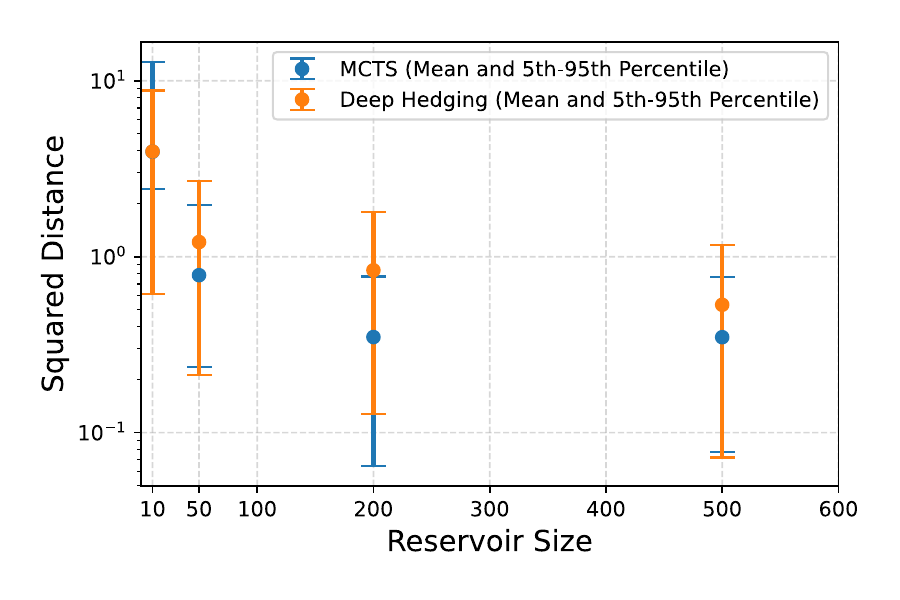}%
    \captionof{figure}{MuZero \& DH learning processes.}%
    \label{fig:small_reservoir_plot}%
\end{minipage}
\subsubsection{Learning portfolio replication in a GBM market with non-convex costs}\label{sec:toyMarketEnvs:gbmCosts}
To evaluate DH and AlphaZero in large state spaces, we revisit the setup of Sec.~\ref{se:learningDeterministic}, replacing the discrete 9-state market model with a continuous Geometric Brownian Motion (GBM) model~\cite{Musiela2005}. The GBM parameters, $\mu = 0.03125$ and $\sigma = 0.25$, are chosen to approximate the bimodal $Q^*$ in Fig.~\ref{fig:reward_plot}. The continuous state space is discretized by rounding asset prices to two decimal places. Additional implementation details are provided in App.~\ref{app:toyMarketEnvs:gbmCosts}. Fig.~\ref{fig:gbmMarket} shows histograms of the action choice at $t_0$ of trained DH and AlphaZero
agents obtained from $100$ independent training cycles.
DH identifies the correct mode in $\sim 13\%$ of cases, as compared to AlphaZero, which identifies the correct mode in $\sim 97\%$ of cases. %
\subsubsection{Learning portfolio replication in a trinomial market with trading constraints}\label{sec:toyMarketEnvs:trinomialWithConstraint}
We consider a setup, where the domain of $Q^*$ is disconnected due to cash constraints. The agent makes five consecutive portfolio adjustments with $\mathcal{A}=\{0,2/40,...,78/40\}$ in a trinomial market~\cite{Musiela2005} with costs $c_k=\min\{12.5\cdot|{\Delta\delta_k}|,2.5\}$. We replace the quadratic loss by an exponential utility $u(x)=-2e^{-0.5x}$ and we impose lower and upper bounds on the cash holdings $b_{min}= 0, b_{max}=8$. Fig.~\ref{fig:brokenQ:a} shows the $Q^*$ (at the initial state) computed using dynamic programming. Fig.~\ref{fig:brokenQ:b} shows a heat map of $Q^*$ (init.~cash $y$-ax., init.~act.~$x$-ax.); white indicates state-action pairs violating constraints, highlighting the non-convexity of the feasible domain. Fig.~\ref{fig:brokenQ:c} shows a histogram of chosen actions. Details are provided in App.~\ref{ap:couterexamples} and~\ref{app:toyMarketEnvs:trinomialWithConstraint}.
\begin{figure}[h]
    \centering%
    \begin{subfigure}[b]{0.5\linewidth}%
        \centering%
        \includegraphics[width=\linewidth]{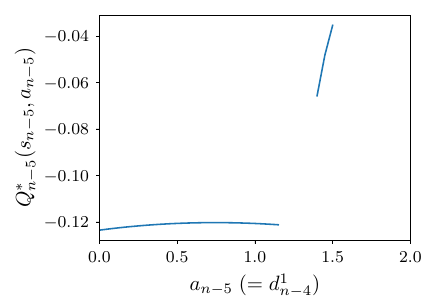}%
        \caption{$Q^*$-function at $t_0$.}%
        \label{fig:brokenQ:a}%
    \end{subfigure}%
    \begin{subfigure}[b]{0.5\linewidth}%
        \centering%
        \includegraphics[width=\linewidth]{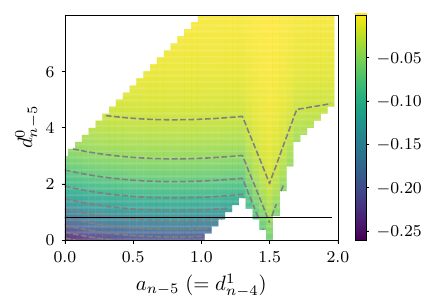}
        \caption{Heat map of $Q^*$ at $t_0$.}%
        \label{fig:brokenQ:b}%
    \end{subfigure}
    \begin{subfigure}[b]{0.5\linewidth}%
    \centering%
        \includegraphics[width=\linewidth]{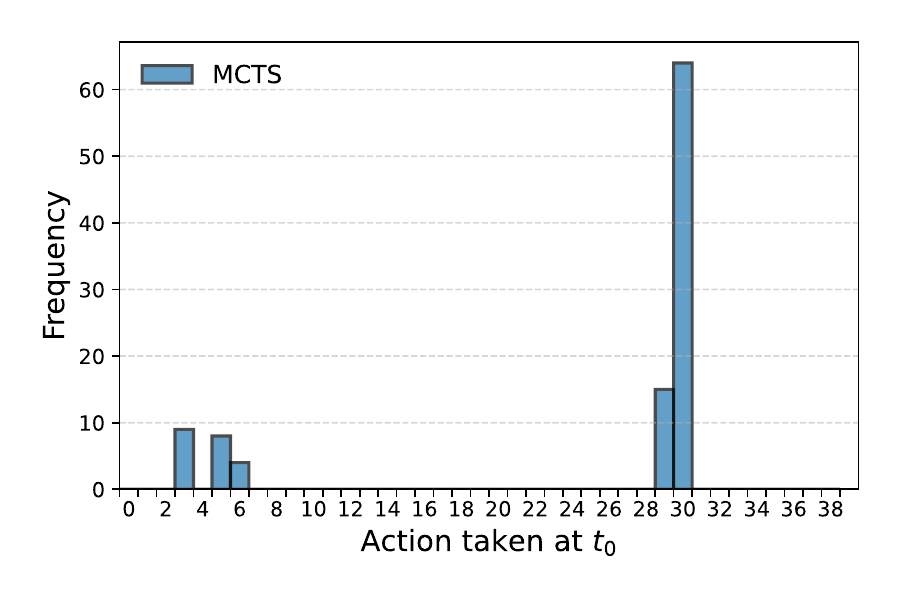}%
        \caption{Hist.~of actions chosen at $t_0$.}%
        \label{fig:brokenQ:c}%
    \end{subfigure}
    \caption{Learning optimal actions in a market, where $Q^*$ has a disconnected domain at init.~state.}
    \label{fig:brokenQ}
\end{figure}
\subsection{Sample efficiency}\label{sec:sampleEfficiency}
One might object that the performance gap observed in Secs.~\ref{sec:minimalEnvs} and~\ref{sec:toyMarketEnvs} arises from an unfair comparison: DH is trained on a reservoir of paths of random market evolution, while AlphaZero samples from the market's transition kernel $\lambda$. It could be argued that AlphaZero has access to more detailed information than DH. To address this, we compare the sample efficiency (see~\cite{szehrcannelli} for a definition) of DH with that of a MuZero variant, see App.~\ref{app:describeMuZero} for a system description and Sec.~\ref{se:financialGamesAndTraining} for MuZero's training procedure. Fig.~\ref{fig:small_reservoir_plot} shows descriptive statistics of terminal losses generated by MuZero and DH agents, trained on reservoirs of market paths with sizes $10, 50, 200, 500$. In these experiments, both systems were trained on exactly the same data. For each reservoir size, we executed $100$ independent learning cycles, sampling a fresh reservoir each time. The agents were trained tabula rasa until they reached their minimal loss. We report the mean, 5th, and 95th percentiles of terminal losses across these trials. MuZero shows a lower sample requirement than DH, consistent with the intuition that MCTS' exploration is targeted, whereas DH's exploration is driven purely by market randomness. For details, refer to App.~\ref{app:sampleEfficiency}.
\section{Conclusions}
Our experiments show that AlphaZero consistently identifies near-optimal replication strategies, outperforming DH in environments with multi-modal $Q^*$. Despite its performance, AlphaZero's high infrastructure complexity limits its scalability, especially in high-dimensional asset spaces, where DH remains more practical. Looking ahead, transformer-based policy architectures may offer a promising middle ground, though their effectiveness in highly stochastic settings remains an open challenge.
Looking ahead, transformer-based policy architectures - such as online decision transformers - offer a promising middle ground, though their effectiveness in highly stochastic settings remains an open challenge~\cite{convergenceAndStabilityStrupl}.

\section*{Acknowledgements}

This work has been supported by UBS Switzerland
AG and its affiliates.

\newpage
\bibliographystyle{plainnat}
\bibliography{main}

\begin{thebibliography}{44}
\providecommand{\natexlab}[1]{#1}
\providecommand{\url}[1]{\texttt{#1}}
\expandafter\ifx\csname urlstyle\endcsname\relax
  \providecommand{\doi}[1]{doi: #1}\else
  \providecommand{\doi}{doi: \begingroup \urlstyle{rm}\Url}\fi

\bibitem[Akiba et~al.(2019)Akiba, Sano, Yanase, Ohta, and Koyama]{optuna}
Takuya Akiba, Shotaro Sano, Toshihiko Yanase, Takeru Ohta, and Masanori Koyama.
\newblock Optuna: A next-generation hyperparameter optimization framework.
\newblock In \emph{Proceedings of the 25th ACM SIGKDD International Conference on Knowledge Discovery \& Data Mining}, KDD '19, page 2623–2631, New York, NY, USA, 2019. Association for Computing Machinery.
\newblock ISBN 9781450362016.
\newblock \doi{10.1145/3292500.3330701}.

\bibitem[Anthony et~al.(2017)Anthony, Tian, and Barber]{Anthony2017}
Thomas Anthony, Zheng Tian, and David Barber.
\newblock Thinking fast and slow with deep learning and tree search.
\newblock In I.~Guyon, U.~Von Luxburg, S.~Bengio, H.~Wallach, R.~Fergus, S.~Vishwanathan, and R.~Garnett, editors, \emph{Advances in Neural Information Processing Systems}, volume~30. Curran Associates, Inc., 2017.

\bibitem[Antonoglou et~al.(2022)Antonoglou, Schrittwieser, Ozair, Hubert, and Silver]{antonoglou2022planning}
Ioannis Antonoglou, Julian Schrittwieser, Sherjil Ozair, Thomas~K Hubert, and David Silver.
\newblock Planning in stochastic environments with a learned model.
\newblock In \emph{International Conference on Learning Representations}, 2022.
\newblock URL \url{https://openreview.net/forum?id=X6D9bAHhBQ1}.

\bibitem[Auer et~al.(2002)Auer, Cesa-Bianchi, and Fischer]{auer2002finite}
Peter Auer, Nicolo Cesa-Bianchi, and Paul Fischer.
\newblock Finite-time analysis of the multiarmed bandit problem.
\newblock \emph{Machine learning}, 47:\penalty0 235--256, 2002.

\bibitem[Barron and Jensen(1990)]{Barron1990}
E.N. Barron and R.~Jensen.
\newblock A stochastic control approach to the pricing of options.
\newblock \emph{Mathematics of Operations Research}, 15\penalty0 (1):\penalty0 49--79, 1990.

\bibitem[Ben-Tal and Nemirovski(2023)]{convexOptim}
Aharon Ben-Tal and Arkadi Nemirovski.
\newblock Convex analysis, nonlinear programming theory, nonlinear programming algorithms, 2023.
\newblock URL \url{https://www2.isye.gatech.edu/\textasciitilde nemirovs/OPTIIILN2023Spring.pdf}.

\bibitem[Bertsekas and Shreve(1996)]{bertsekas1996stochastic}
Dimitri Bertsekas and Steven~E Shreve.
\newblock \emph{Stochastic optimal control: the discrete-time case}, volume~5.
\newblock Athena Scientific, 1996.

\bibitem[Black and Scholes(1973)]{black1973pricing}
Fischer Black and Myron Scholes.
\newblock The pricing of options and corporate liabilities.
\newblock \emph{Journal of political economy}, 81\penalty0 (3):\penalty0 637--654, 1973.

\bibitem[Browne et~al.(2012)Browne, Powley, Whitehouse, Lucas, Cowling, Rohlfshagen, Tavener, Perez, Samothrakis, and Colton]{Browne2014}
Cameron~B. Browne, Edward Powley, Daniel Whitehouse, Simon~M. Lucas, Peter~I. Cowling, Philipp Rohlfshagen, Stephen Tavener, Diego Perez, Spyridon Samothrakis, and Simon Colton.
\newblock A survey of monte carlo tree search methods.
\newblock \emph{IEEE Transactions on Computational Intelligence and AI in Games}, 4\penalty0 (1):\penalty0 1--43, 2012.
\newblock \doi{10.1109/TCIAIG.2012.2186810}.

\bibitem[Buehler et~al.(2019{\natexlab{a}})Buehler, Gonon, Teichmann, and Wood]{buehler2019deep}
H.~Buehler, L.~Gonon, J.~Teichmann, and B.~Wood.
\newblock Deep hedging.
\newblock \emph{Quantitative Finance}, 19\penalty0 (8):\penalty0 1271--1291, 2019{\natexlab{a}}.
\newblock \doi{10.1080/14697688.2019.1571683}.

\bibitem[Buehler et~al.(2019{\natexlab{b}})Buehler, Gonon, Teichmann, Wood, Mohan, and Kochems]{buehler2019deep2}
H.~Buehler, L.~Gonon, J.~Teichmann, B.~Wood, B.~Mohan, and J.~Kochems.
\newblock Deep hedging: hedging derivatives under generic market frictions using reinforcement learning.
\newblock Technical report, Swiss Finance Institute, 2019{\natexlab{b}}.

\bibitem[Cannelli et~al.(2023)Cannelli, Nuti, Sala, and Szehr]{szehrcannelli}
L.~Cannelli, G.~Nuti, M.~Sala, and O.~Szehr.
\newblock Hedging using reinforcement learning: Contextual $k$-armed bandit versus $q$-learning.
\newblock \emph{The Journal of Finance and Data Science, in print}, 2023.
\newblock \doi{https://doi.org/10.1016/j.jfds.2023.100101}.

\bibitem[Cao et~al.(2021)Cao, Chen, Hull, and Poulos]{cao2019deep}
J.~Cao, J.~Chen, J.~Hull, and Z.~Poulos.
\newblock Deep hedging of derivatives using reinforcement learning.
\newblock \emph{The Journal of Financial Data Science}, 3\penalty0 (1):\penalty0 10--27, 2021.
\newblock \doi{10.3905/jfds.2020.1.052}.

\bibitem[Chang et~al.(2005)Chang, Fu, Hu, and Marcus]{Chang2005}
Hyeong~Soo Chang, Michael~C. Fu, Jiaqiao Hu, and Steven~I. Marcus.
\newblock An adaptive sampling algorithm for solving markov decision processes.
\newblock \emph{Operations Research}, 53\penalty0 (1):\penalty0 126--139, 2005.
\newblock \doi{10.1287/opre.1040.0145}.

\bibitem[Curin et~al.(2021)Curin, Kettler, Kleisinger-Yu, Komaric, Krabichler, Teichmann, and Wutte]{Curin2021}
Nicolas Curin, Michael Kettler, Xi~Kleisinger-Yu, Vlatka Komaric, Thomas Krabichler, Josef Teichmann, and Hanna Wutte.
\newblock A deep learning model for gas storage optimization.
\newblock \emph{Decisions in Economics and Finance}, 44\penalty0 (2):\penalty0 1021–1037, 2021.
\newblock \doi{10.1007/s10203-021-00363-6}.

\bibitem[Duffie(2001)]{duffie2001}
D.~Duffie.
\newblock \emph{Dynamic Asset Pricing Theory: Third Edition}.
\newblock Princeton Univ. Press, 2001.

\bibitem[Guo et~al.(2014)Guo, Singh, Lee, Lewis, and Wang]{guo2014deep}
Xiaoxiao Guo, Satinder Singh, Honglak Lee, Richard~L Lewis, and Xiaoshi Wang.
\newblock Deep learning for real-time atari game play using offline monte-carlo tree search planning.
\newblock In Z.~Ghahramani, M.~Welling, C.~Cortes, N.~Lawrence, and K.Q. Weinberger, editors, \emph{Advances in Neural Information Processing Systems}, volume~27. Curran Associates, Inc., 2014.

\bibitem[Hodges and Neuberger(1989)]{hodges1989option}
S.~Hodges and A.~Neuberger.
\newblock Option replication of contingent claims under transactions costs.
\newblock \emph{The Review of Futures Markets}, 2\penalty0 (8):\penalty0 222--239, 1989.

\bibitem[Karoui and Quenez(1995)]{ElKaroui1995}
Nicole~El Karoui and Marie-Claire Quenez.
\newblock Dynamic programming and pricing of contingent claims in an incomplete market.
\newblock \emph{{SIAM} Journal on Control and Optimization}, 33\penalty0 (1):\penalty0 29--66, 1995.
\newblock \doi{10.1137/s0363012992232579}.

\bibitem[Klein(2022)]{NNChess}
Dominik Klein.
\newblock Neural networks for chess, 2022.
\newblock URL \url{https://arxiv.org/abs/2209.01506}.

\bibitem[Kocsis et~al.(2006)Kocsis, Szepesv{\'{a}}ri, and Willemson]{Kocsis2006_}
L.~Kocsis, C.~Szepesv{\'{a}}ri, and J.~Willemson.
\newblock {Improved Monte-Carlo Search}.
\newblock Technical Report~1, Univ. Tartu, Estonia, 2006.

\bibitem[Kocsis and Szepesv{\'{a}}ri(2006)]{Kocsis2006}
Levente Kocsis and Csaba Szepesv{\'{a}}ri.
\newblock Bandit based monte-carlo planning.
\newblock In \emph{Lecture Notes in Computer Science}, pages 282--293. Springer Berlin Heidelberg, 2006.
\newblock \doi{10.1007/11871842\textunderscore29}.

\bibitem[Krabichler and Teichmann(2023)]{krablicher1}
Thomas Krabichler and Josef Teichmann.
\newblock A case study for unlocking the potential of deep learning in asset-liability-management.
\newblock \emph{Front. Artif. Intell.}, 2023.

\bibitem[Musiela and Rutkowski(2005)]{Musiela2005}
Marek Musiela and Marek Rutkowski.
\newblock \emph{Martingale Methods in Financial Modelling}.
\newblock Springer Berlin Heidelberg, 2005.
\newblock ISBN 9783540266532.
\newblock \doi{10.1007/b137866}.

\bibitem[P{\'{e}}ret and Gar{\c{c}}ia(2004)]{peret2004}
Laurent P{\'{e}}ret and Fr{\'{e}}d{\'{e}}rick Gar{\c{c}}ia.
\newblock On-line search for solving markov decision processes via heuristic sampling.
\newblock In \emph{Proceedings of the 16th Eureopean Conference on Artificial Intelligence, ECAI'2004, including Prestigious Applicants of Intelligent Systems, {PAIS} 2004, Valencia, Spain, August 22-27, 2004}, pages 530--534. {IOS} Press, 2004.

\bibitem[Perez-Liebana et~al.(2019)Perez-Liebana, Lucas, Gaina, Togelius, Khalifa, and Liu]{gvgaibook2019}
Diego Perez-Liebana, Simon~M. Lucas, Raluca~D. Gaina, Julian Togelius, Ahmed Khalifa, and Jialin Liu.
\newblock \emph{General Video Game Artificial Intelligence}, volume~3.
\newblock Morgan \& Claypool Publishers, 2019.

\bibitem[Pollard(2001)]{pollard2001user}
David Pollard.
\newblock \emph{A User's Guide to Measure Theoretic Probability}.
\newblock Cambridge Series in Statistical and Probabilistic Mathematics. Cambridge University Press, 2001.
\newblock \doi{10.1017/CBO9780511811555}.

\bibitem[Puterman(2014)]{puterman2014markov}
Martin~L Puterman.
\newblock \emph{Markov decision processes: discrete stochastic dynamic programming}.
\newblock John Wiley \& Sons, 2014.

\bibitem[Romstad(2018)]{chess}
Tord Romstad.
\newblock Alphazero versus stockfish, 2018.
\newblock URL \url{\footnotesize{https://www.chess.com/news/view/alphazero-reactions-from-top-gms-stockfish-author}}.

\bibitem[Schadd et~al.(2008)Schadd, Winands, van~den Herik, Chaslot, and Uiterwijk]{Schadd}
Maarten P.~D. Schadd, Mark H.~M. Winands, H.~Jaap van~den Herik, Guillaume M. J.~B. Chaslot, and Jos W. H.~M. Uiterwijk.
\newblock Single-player monte-carlo tree search.
\newblock In \emph{Computers and Games}, pages 1--12. Springer Berlin Heidelberg, 2008.
\newblock \doi{10.1007/978-3-540-87608-3\textunderscore1}.

\bibitem[Schrittwieser et~al.(2020{\natexlab{a}})Schrittwieser, Antonoglou, Hubert, Simonyan, Sifre, Schmitt, Guez, Lockhart, Hassabis, Graepel, Lillicrap, and Silver]{schrittwieser2019mastering}
Julian Schrittwieser, Ioannis Antonoglou, Thomas Hubert, Karen Simonyan, Laurent Sifre, Simon Schmitt, Arthur Guez, Edward Lockhart, Demis Hassabis, Thore Graepel, Timothy Lillicrap, and David Silver.
\newblock {Mastering Atari, Go, Chess and Shogi by planning with a learned model}.
\newblock \emph{Nature}, 588\penalty0 (7839):\penalty0 604--609, 2020{\natexlab{a}}.
\newblock \doi{10.1038/s41586-020-03051-4}.

\bibitem[Schrittwieser et~al.(2020{\natexlab{b}})Schrittwieser, Antonoglou, Hubert, Simonyan, Sifre, Schmitt, Guez, Lockhart, Hassabis, Graepel, Lillicrap, and Silver]{schrittwieser2020mastering}
Julian Schrittwieser, Ioannis Antonoglou, Thomas Hubert, Karen Simonyan, Laurent Sifre, Simon Schmitt, Arthur Guez, Edward Lockhart, Demis Hassabis, Thore Graepel, Timothy Lillicrap, and David Silver.
\newblock {Mastering Atari, Go, chess and shogi by planning with a learned model}.
\newblock \emph{Nature}, 588\penalty0 (7839):\penalty0 604--609, 2020{\natexlab{b}}.
\newblock ISSN 0028-0836.
\newblock \doi{10.1038/s41586-020-03051-4}.

\bibitem[Shafer and Vovk(2001)]{Shafer2001}
Glenn Shafer and Vladimir Vovk.
\newblock \emph{Probability and Finance: It's Only a Game!}
\newblock Wiley, 2001.
\newblock \doi{10.1002/0471249696}.

\bibitem[Silver et~al.(2017)Silver, Schrittwieser, Simonyan, Antonoglou, Huang, Guez, Hubert, Baker, Lai, and Bolton]{silver2017mastering}
D.~Silver, J.~Schrittwieser, K.~Simonyan, I.~Antonoglou, A.~Huang, A.~Guez, T.~Hubert, L.~Baker, M.~Lai, and A.~Bolton.
\newblock {Mastering the game of Go without human knowledge}.
\newblock \emph{Nature}, 550\penalty0 (7676):\penalty0 354--359, 2017.
\newblock \doi{https://doi.org/10.1038/nature24270}.

\bibitem[Silver et~al.(2014)Silver, Lever, Heess, Degris, Wierstra, and Riedmiller]{silver2014deterministic}
David Silver, Guy Lever, Nicolas Heess, Thomas Degris, Daan Wierstra, and Martin Riedmiller.
\newblock Deterministic policy gradient algorithms.
\newblock In \emph{International conference on machine learning}, pages 387--395. Pmlr, 2014.

\bibitem[Silver et~al.(2016)Silver, Huang, Maddison, Guez, Sifre, van~den Driessche, Schrittwieser, Antonoglou, Panneershelvam, Lanctot, Dieleman, Grewe, Nham, Kalchbrenner, Sutskever, Lillicrap, Leach, Kavukcuoglu, Graepel, and Hassabis]{silver2016mastering}
David Silver, Aja Huang, Chris~J. Maddison, Arthur Guez, Laurent Sifre, George van~den Driessche, Julian Schrittwieser, Ioannis Antonoglou, Veda Panneershelvam, Marc Lanctot, Sander Dieleman, Dominik Grewe, John Nham, Nal Kalchbrenner, Ilya Sutskever, Timothy Lillicrap, Madeleine Leach, Koray Kavukcuoglu, Thore Graepel, and Demis Hassabis.
\newblock Mastering the game of go with deep neural networks and tree search.
\newblock \emph{Nature}, 529\penalty0 (7587):\penalty0 484--489, 2016.
\newblock \doi{10.1038/nature16961}.

\bibitem[Sutton and Barto(2018)]{sutton2018reinforcement}
R.~S. Sutton and A.~G. Barto.
\newblock \emph{Reinforcement learning: An introduction}.
\newblock The MIT Press, 2018.
\newblock URL \url{http://incompleteideas.net/book/the-book-2nd.html}.

\bibitem[{\'{S}}wiechowski et~al.(2022){\'{S}}wiechowski, Godlewski, Sawicki, and Ma{\'{n}}dziuk]{wiechowski2022}
Maciej {\'{S}}wiechowski, Konrad Godlewski, Bartosz Sawicki, and Jacek Ma{\'{n}}dziuk.
\newblock Monte carlo tree search: a review of recent modifications and applications.
\newblock \emph{Artificial Intelligence Review}, 56\penalty0 (3):\penalty0 2497--2562, 2022.
\newblock \doi{10.1007/s10462-022-10228-y}.

\bibitem[Szehr(2023)]{Szehr2023}
Oleg Szehr.
\newblock Hedging of financial derivative contracts via {Monte Carlo} tree search.
\newblock \emph{Journal of Computational Finance}, 2023.
\newblock ISSN 1755-2850.
\newblock \doi{10.21314/jcf.2023.009}.

\bibitem[Teichmann(2020)]{Teich00}
J.~Teichmann.
\newblock {Deep Hedging}, 2020.
\newblock URL \url{https://gist.github.com/jteichma}.

\bibitem[Vanderbei(1996)]{vanderbei1996}
R.~Vanderbei.
\newblock \emph{Optimal sailing strategies, statistics and operations research program}.
\newblock University of Princeton, 1996.

\bibitem[Xiao et~al.(2019)Xiao, Huang, Mei, Schuurmans, and M\"{u}ller]{maxEntroPlanning}
Chenjun Xiao, Ruitong Huang, Jincheng Mei, Dale Schuurmans, and Martin M\"{u}ller.
\newblock Maximum entropy monte-carlo planning.
\newblock In H.~Wallach, H.~Larochelle, A.~Beygelzimer, F.~d\textquotesingle Alch\'{e}-Buc, E.~Fox, and R.~Garnett, editors, \emph{Advances in Neural Information Processing Systems}, volume~32. Curran Associates, Inc., 2019.

\bibitem[Zakamouline(2006)]{Zakamouline2006}
Valeri~I. Zakamouline.
\newblock European option pricing and hedging with both fixed and proportional transaction costs.
\newblock \emph{Journal of Economic Dynamics and Control}, 30\penalty0 (1):\penalty0 1--25, 2006.
\newblock \doi{10.1016/j.jedc.2004.11.002}.

\bibitem[Štrupl et~al.(2025)Štrupl, Szehr, Faccio, Ashley, Srivastava, and Schmidhuber]{convergenceAndStabilityStrupl}
Miroslav Štrupl, Oleg Szehr, Francesco Faccio, Dylan~R. Ashley, Rupesh~Kumar Srivastava, and Jürgen Schmidhuber.
\newblock On the convergence and stability of upside-down reinforcement learning, goal-conditioned supervised learning, and online decision transformers, 2025.
\newblock URL \url{https://arxiv.org/abs/2502.05672}.

\end{thebibliography}

\newpage
\appendix
\section{Details on algorithms}

\subsection{From UCT to AlphaZero}\label{app:describeAlphaZero}
UCT is a variant of MCTS that employs the UCB1 bandit policy~\cite{auer2002finite} to guide exploration and exploitation within the planning tree. The UCB1 policy selects actions by maximizing an upper confidence bound on the estimated reward:
\begin{equation}
UCB1_a = \bar{R}_a+wc_{N,N_a}\ \textnormal{with}\ c_{N,N_a}=\sqrt{\frac{\ln(N)}{N_a}},
\label{eq:UCB1}
\end{equation}
where $\bar{R}_a$ is the average reward obtained from selecting action $a$, $N$ is the total number of action selections, and $N_a$ is the number of times action $a$ has been taken. The term $\bar{R}_a$ promotes exploitation of actions with high expected rewards, while $c_{N, N_a}$ encourages exploration of less-visited actions. The scalar $w$ is a problem-specific hyperparameter that balances these two objectives.

AlphaZero is an MCTS variant designed for adversarial game tree search~\cite{Anthony2017, silver2017mastering}. It enhances UCT by incorporating NNs, which are typically used to model both the tree policy and the value function. These NNs serve two key purposes: \emph{(1)} improving the efficiency of tree construction by providing better action priors, and \emph{(2)} estimating the value of leaf nodes, thereby reducing the reliance on Monte Carlo estimation. Augmenting UCT with NNs comes at the cost of losing UCT's theoretical guarantee of convergence to optimal actions. But it offers two key advantages. First, NNs enable faster estimation of state values as compared to full tree evaluation. Second, they provide generalization: whereas UCT requires costly simulations to evaluate unseen states, a trained NN can leverage similarity between states to estimate values efficiently.

Imitation learning aims to train an apprentice policy $\pi^A$ to mimic an expert policy $\pi^E$, typically through supervised learning on a dataset of expert-generated state-action pairs. In AlphaZero, the expert's role is taken by a UCT variant, while the apprentice is represented by a NN. The apprentice network comprises two interrelated components: \emph{1)} The policy network is trained on tree statistics targets generated by the tree policy. Writing $N_{s,a}$ for the number of times action $a$ has been chosen from state $s$ during MCTS search, and $N_{s} = \sum_a N_{s,a}$ for the total number of visits to $s$, the supervised loss for this task is:
\begin{align*}
Loss_{TPT}=-\sum_{a}\frac{N_{s,a}}{N_s}\log{\pi^{A}(a|s)}.
\end{align*}
In turn a modified UCB1 score incorporates the network's prior to guide tree search towards stronger actions at each state by maximizing:
\begin{align*}
AlphaZero\_ UCB1_a=\hat{Q}(s,a)+w\pi^A(a|s)\frac{\sqrt{\ln(N_s)}}{N_{s,a}+1},
\end{align*}
where $\hat{Q}(s,a)$ is an estimate of action-values (computed as the  mean reward $\bar{R}_a$ from $s$).
\emph{2)} The value network estimates the expected game outcome $V^A(s)$ from a given state $s$. MCTS provides a scalar training target $z$ (e.g., the final game result), and the supervised loss is:
\begin{align*}
Loss_{V}=-(z-V^A(s))^2.
\end{align*}
When a simulation reaches a terminal state, the resulting trajectory is stored in a replay buffer. Alongside each state, the buffer also records the observed reward, the network’s value estimate, and the action probabilities generated during tree search. The network is updated via stochastic gradient descent by aligning its outputs with these recorded targets. To regularize learning and encourage efficient representation sharing, both the policy and value networks are typically trained jointly via a multitask architecture, whose loss is simply the sum of $Loss_V$ and $Loss_{TPT}$. While AlphaZero is originally trained through self-play — where one instance competes against another — in portfolio replication, we treat AlphaZero as a pure RL algorithm trained through interaction with a market environment rather than via self-play.
\subsection{MuZero}\label{app:describeMuZero}

MuZero~\cite{schrittwieser2020mastering} is an extension of AlphaZero that eliminates the need for an explicit environment simulator. Instead of sampling from a known transition model $\tkernel$, MuZero learns an internal dynamics model and uses this model to simulate state-action transitions in a manner similar to AlphaZero's usage of $\tkernel$. MuZero is best described as a latent-model-based reinforcement learning model, as it performs planning using a latent representation of the environment rather than directly interacting with the true environment. In its original formulation, MuZero introduces three key components that augment the AlphaZero architecture:

\begin{itemize}
\item \emph{Representation function}: $h$ maps an environment state $s_t$ (or observation history in the partially observable case) to a latent state $\tilde{s}_t = h(s_t)$. This latent representation facilitate planning by reducing the computational and memory burden associated with operating directly on complex and high-dimensional states.

\item \emph{Dynamics function}: Given a latent state $\tilde{s}_t$ and action $a_t$, $g$ predicts the next latent state and an intermediate reward signal: $(\tilde{s}_{t+1}, \tilde{r}_{t+1}) = g(\tilde{s}_t, a_t)$. This function is learned from observation and takes the role played by $\tkernel$ in AlphaZero.

\item \emph{Prediction function}: Given a latent state $\tilde{s}_t$, the function $f$ outputs a value estimate and a policy distribution over actions: $(v_t, p_t) = f(\tilde{s}_t)$.
\end{itemize}

In our experiments with MuZero, we consider a modified setting. One simplification is that the environment states are sufficiently low-dimensional and structured such that a separate representation function is not required, $h=\mathrm{id}_{\mathcal{S}}$. Moreover, the absence of market impact implies that actions do not affect the evolution of the market state. The market is thus modeled as a process $X_k$ evolving independently of actions and other state components, with dynamics governed by a Markov kernel $\lambda_x(x_{k+1}|x_k)$.  Accordingly, MuZero learns only the dynamic function $g_x$ - an estimate of $\lambda_x$ - since the transitions of the remaining state components, given the market values, are known and deterministic.
To account for the learning of the dynamic function MuZero employs a modification of AlphaZero's tree policy maximizing:
\begin{align*}
MuZero\_ UCB1_a=\hat{Q}(s,a)+\pi^A(a|s)\frac{\sqrt{\ln(N_s)}}{N_{s,a}+1}\left(w_1+\ln\left(\frac{N_s+w_2+1}{w_2}\right)\right),
\end{align*}
where the constants $w_1,w_2$ weight the influence of the prior $\pi^A$ over that of $Q(s,a)$ as number of visits in $s$ increases. MuZero stores sequences of observations, action probabilities, and rewards in a replay buffer. During training, MuZero samples sequences from the buffer and for each sampled sequence, the networks are trained to match predicted values, rewards, and policies to the observed targets. Compared to AlphaZero, MuZero introduces an additional loss term corresponding to the prediction of immediate rewards. The total loss combines value, reward, and policy losses across all unroll steps, along with an optional regularization term on model parameters. By learning both dynamics and planning simultaneously, MuZero generalizes AlphaZero’s approach to a broader class of decision-making problems, including those without known simulators. In our implementation, however, the market dynamics evolve independently of actions and other state components and are governed by a given Markov kernel. As a result, MuZero only learns an estimate of this kernel, and no exploration is required in learning the market dynamics model.

\section{Details on Mathematical Derivations and Proofs}
\subsection{On the Existence of a Borel-Measurable Optimal Policy}
\label{ap:measurableoptpi}

In this appendix, we provide additional details for the interested reader regarding the gaps $\inf_{\pi\in\Pi^c} \|V^*-V^{\pi}\|_{\infty} > 0$ and $\inf_{\pi\in\Pi^c} \|Q^*-Q^{\pi}\|_{\infty} > 0$, cf.~Sec.~\ref{se:backgroundAndSetting}. These were omitted from the main text due to their technical nature. RL problems with continuous state and action spaces discussed are treated in full generality in \cite{bertsekas1996stochastic}, which shows that in many RL problems, no Borel measurable optimal policies exist—even when rewards and transition kernels are Borel measurable. To address this, \cite{bertsekas1996stochastic} introduces the broader class of universally measurable policies, where optimal policies can be found. We restricted the RL problems under consideration to avoid obscuring central arguments that focus on relatively straightforward convexity arguments and curve properties with technical measurability and topology considerations.  Nonetheless, to offer a refined perspective, we briefly illustrate the lines of reasoning encountered in more general settings. Specifically, recall that if the action space $\mathcal{A}$ is compact and convex, rewards are continuous, and the transition kernel is weakly continuous, then the state- and action-value functions are well-defined and continuous. Assuming for simplicity that $\mathcal{A} \subset \mathbb{R}$, we show that there exists a lower semi-continuous (LSC) deterministic optimal policy $\pi^* : \mathcal{S} \rightarrow \mathcal{A}$.

\begin{lemma}
\label{le:infargmaxLSC}
Let $(X,d)$ be a metric space, $Y\subset\mathbb{R}$ a compact subset and $Q:X\times Y \rightarrow \mathbb{R}$ a continuous function. Then the function defined as $f(x) := \inf \argmax Q(x,\cdot)$ for all $x \in X$ is lower semi-continuous.
\end{lemma}
Since any {bounded} LSC function is the pointwise limit of bounded Lipschitz functions~\cite{pollard2001user}, the policy $\pi^*$ is Borel measurable. Therefore, within our restricted - yet practically relevant - setting, invoking universal measurability offers no additional benefit.
Note that no the gaps are present if $Q^*(s,\cdot)$ is unimodal for all $s\in\mathcal{S}$. Moreover if $|\argmax Q^*(s,\cdot)|=1$ for all $s\in\mathcal{S}$ then $\pi^* \in \Pi^c$. However, when $Q^*(s,\cdot)$ is multimodal, value gaps may persist. These gaps can sometimes be eliminated by imposing additional assumptions - typically domain-specific in nature. We found that restricting
to price process kernels $\lambda_x$ that are absolutely continuous with respect to Lebesgue measure together with compactness on $\mathcal{S}$ still allows for a wide range of practically relevant replication domains in which $\inf_{\pi \in \Pi^c} \|Q^*-Q^{\pi}\|_{\infty} = 0$. Finally, if one does not require uniform approximation over the entire domain, classical results such as Luzin's theorem and its variants may be applied to justify approximation on large subsets of $\mathcal{S}$. However, a key challenge remains: NNs with continuous activation functions can only represent continuous functions, while $\pi^*$, though Borel measurable and LSC, may not be continuous. This creates a mismatch when approximating optimal policies.

\begin{proof}[Proof of Lem.~\ref{le:infargmaxLSC}]
Recall that $f:(X,d) \rightarrow \mathbb{R}$ is LSC if and only if for each
$t\in\mathbb{R}$ the pre-image $f^{-1}((t,+\infty))$ is open. Fix $t$ and $x \in f^{-1}((t,+\infty)) = \{ x \mid f(x) > t\}$. It suffice to show that there exists an open neighborhood of $x$ which is also
contained in the pre-image. Since $X$ is a metric space, {this is equivalent to the following:} for each sequence $(x_n)$,
$x_n \rightarrow x$ there exists $n_0$ such that for all $n  > n_0$,
$f(x_n) > t$. It suffice to show that for each sequence $x_n \rightarrow x$
the accumulation points of $(f(x_n))$ are not smaller than $f(x)$.
Fix a sequence $x_n \rightarrow x$. For the sake of contradiction
assume that $(f(x_n))$ has a limit point $a < f(x) = \inf \argmax  [Q(x,\cdot)]$. {Without loss of generality,} denote the
subsequence of $(x_n)$, which has the limit $a$
also by $(x_n)$. Since $f(x_n) = \inf \argmax [Q(x_n,\cdot)]$ there exist a sequence
$(a_n)$, $a_n \in \argmax [Q(x_n,\cdot)]$ such that $a_n \rightarrow a$.
Since $x \mapsto \max Q(x,\cdot)$ is continuous we obtain
$Q(x_n,a_n) = \max Q(x_n,\cdot) \rightarrow \max Q(x,\cdot)$.
Further since $Q$ is continuous we obtain
$Q(x_n,a_n) \rightarrow Q(x,a)$. Taken together this yields
$Q(x,a) = \max Q(x,\cdot)$, i.e., $a \in \argmax [Q(x,\cdot)]$.
But $a < f(x) = \inf \argmax [Q(x,\cdot)]$ a contradiction.
\end{proof}

\subsection{Proofs of Lemmas in Convex Optimization}~\label{app:convexLemmas}
The proof of Thm.~\ref{thm:convex} uses the following two elementary lemmas.
\concavecomp*

\begin{proof}
During this proof we will write $x=(x_1,x_2)$ for the argument of $f$, where $x_1 \in \mathbb{R}$ stands 
for the first argument of $f$ and $x_2$ for the rest. Similarly, we write $y=(y_1,y_2)$ for the argument of $g$. Let $(y,x_2),(y',x_2') \in \mathbb{R}^{m+n-1}$, $\alpha \in [0,1]$. Using first that $g$ is concave and that $f$ is increasing and then that $f$ is concave, we find
\begin{align*}
f\circ (g\times \mathrm{id}_{\mathbb{R}^{m-1}})
(\alpha(y,x_2)+(1-\alpha)(y',x_2'))
&=
f(g(\alpha y+(1-\alpha) y'),\alpha x_2+(1-\alpha)x_2') 
\\
&\geq 
f( \alpha g(y)+(1-\alpha)g(y'), \alpha x_2+(1-\alpha)x_2')
\\
&=
f( \alpha (g(y),x_2)+(1-\alpha)(g(y'),x_2'))
\\
&\geq
\alpha f(g(y),x_2)+(1-\alpha)f(g(y'),x_2').
\end{align*}
This shows that $f\circ (g\times \mathrm{id}_{\mathbb{R}^{m-1}})$ is concave. To prove that $f\circ (g\times \mathrm{id}_{\mathbb{R}^{m-1}})$ is increasing in the first argument let us fix $(y_2,x_2) \in \mathbb{R}^{m+n-2}$ and $y_1 < y_1'\in \mathbb{R}$. Since
$g$ is increasing in the first argument we get $g(y_1,y_2) < g(y_1',y_2)$ and further since $f$ is increasing in the first argument we get $f(g(y_1,y_2),x_2) < f(g(y_1',y_2),x_2)$.
\end{proof}

\concavemax*
\begin{proof}
Fix $x',x'' \in \mathbb{R}^n$ and 
$\alpha \in [0,1]$. Fix $y' \in \argmax_{y} f(x',y)$ and $y'' \in \argmax_{y} f(x'',y)$.
Define $(x''',y''') := \alpha (x',y') + (1-\alpha) (x'',y'')$.
From concavity of $f$ and definition of $g$ we have
$$
g(x''') = \max_{y} f(x''',y)  \geq f(x''',y''') \geq \alpha f(x',y') + (1-\alpha) f(x'',y'') = \alpha g(x') + (1-\alpha) g(x''),
$$
which proves that $g$ is concave.
Now assume that $f$ has the described increasing property and let
$x_1 < x_1' \in \mathbb{R}$, $(x_2,y) \in \mathbb{R}^{n+m-1}$ then
$$
g(x_1,x_2) = \max_y f(x_1,x_2,y) \leq \max_y f(x_1',x_2,y) = g(x_1',x_2).
$$
\end{proof}

\subsection{Proof of Equivalence of Unimodality}
\label{ap:equivalence}

To prove point \emph{1.} of Thm.~\ref{thm:unimodal}, we made the simplifying assumption $\pi^* \in \Pi^c$. Here we explain when this assumption is satisfied and how the general case $\pi^* \notin \Pi^c$ is treated.
First note that if $Q^{*}(s,\cdot)$ is unimodal for each $s\in \mathcal{S}$ and when $|\argmax Q(s,\cdot)|=1$ for all $s\in\mathcal{S}$ then $\pi^* \in \Pi^c$.
In other words, it is sufficient to assume apart from unimodality that maxima are non-degenerate, which covers many situations of practical interest. In the general case $\pi^* \notin \Pi^c$ one can prove relaxed but still practical variations of the result.
For example, if $\pi_0\in \Pi^c$, $\pi^*\notin \Pi^c$
one can approximate $\pi^*$ with $\pi^*_{\epsilon}$ so that
$\|Q^*-Q^{\pi^*_\epsilon}\|_{\infty} < \epsilon$ for some suitably
small $\epsilon >0$. In the light of the proof of Thm.~\ref{thm:unimodal},
$Q^{\pi^*_\epsilon}$ has to stay on non-decreasing curves as $Q^*$ (due to its unimodality)
does. This yields non-decreasing curves up to an $\epsilon$ error and results in a slightly relaxed statement: for each
non-optimal $\pi_0 \in \Pi^c$ and $\epsilon>0$ there exists $\pi^*_{\epsilon} \in \Pi^c$, where $\|Q^*-Q^{\pi^*_\epsilon}\|_{\infty} < \epsilon$, and an up-to-$\epsilon$-non-decreasing curve connecting
$\pi_0$ with $\pi^*_{\epsilon}$. This property is useful from an algorithm design perspective as stochastic gradient methods could ignore $\epsilon$ bumps.

\section{Counterexamples and the Necessity of Assumptions for the Proof of Convexity}
\label{ap:couterexamples}

This illustrates the necessity of the assumptions underlying Thm.~\ref{thm:convex} through a series of experiments. Specifically, we demonstrate that the violation of either the increasing {concave} utility assumption or the convex cost assumption results in the emergence of multimodality in $Q^*(s,\cdot)$ for some state $s$. These findings are presented in Sec.~\ref{ap:nonincutil} and~\ref{ap:nonconvexcost}, respectively. Further, in the Sec.~\ref{ap:boundsconstraint}, we show that introducing bounds on the cash state variable gives rise to a $Q^*(s,\cdot)$ with a disconnected domain of definition, corresponding to a fragmented set of available actions. This fragmentation produces effects analogous to those caused by multimodality: local policy search methods are prone to becoming trapped within isolated regions of the action space.

For brevity we use a slight \lq\lq{}abuse of notation\rq\rq{}. We will omit the last 
state component ($W_k$ in the proof of Thm.~\ref{thm:convex}) and assume
that the costs are absorbed in the \lq\lq{}cash\rq\rq{} random variable $\delta^0_k$.
For brevity we describe the state by $S_k = (k,\delta_k, X_k)$, where the self-financing constraint entails that $\delta_k^0$ is updated according to
\begin{equation}
\delta_{k+1}^0 = \delta_{k}^0 - \delta_k^{1:\adim}X_k^{1:\adim} - c_k(\delta_{k+1}^{1:\adim}-\delta_k^{1:\adim}, X_k^{1:\adim}) =: g_k(\delta_k,\delta_{k+1}^{1:\adim},X_k^{1:\adim})
\label{eq:chashupdate}
\end{equation}

Considering the cash together with the accumulated costs 
in one state component or split across several state components
does not change the problem. In sense of proof of theorem
\ref{thm:convex} the changes are minimal (one suffices with one application of Lem.~\ref{le:concavecomp} in point (i) instead of original two applications).
The theorem holds for both formulations mutatis mutandis.

\subsection{Non-increasing Utility}
\label{ap:nonincutil}

Here we will restrict to single price process $X_k = X_k^{1}$ ($\adim=1$) with only
two states $\{1,2\}$ with probability staying in the same state $0.8$
and transitioning to the other state $0.2$. The action $A_k = \delta_{k+1}^{1}$ (number of shares) can take 20 values uniformly spaced between 0 and 1, precisely $\mathcal{A} = \{0.0, 0.05, \ldots 0,95\}$.
To violate increasing utility assumption of the theorem we choose
the quadratic utility $u(x) = -x^2$, which is still concave.
As convex cost we assume quadratic cost $c(\delta^1_{k+1}-\delta^1_{k}, X_k) = 0.5|(\delta^1_{k+1}-\delta^1_{k}) X_k|^2$. The liability has the form $Z_n = (X_n-K)^+$, where $K=2$.

The figure \ref{fig:couterutilcost:a} shows the optimal action-value function $Q^*(s_{n-1},a_{n-1})$ as function of action $a_{n-1}= \deval_{n}^1$
for the state $s_{n-1} = (n-1, \deval_{n-1},x_{n-1})$, where the cash is $\deval_{n-1}^0 = -0.6$, number of shares held is $\deval_{n-1}^1 = 0.55$ and the price is $x_{n-1}=2.0$.
The bi-modality in action $a_{n-1} = \deval_{n}^1$ is clearly apparent.

\begin{figure}[h]
    \begin{subfigure}[b]{0.5\linewidth}%
        \centering%
        \includegraphics[width=\textwidth]{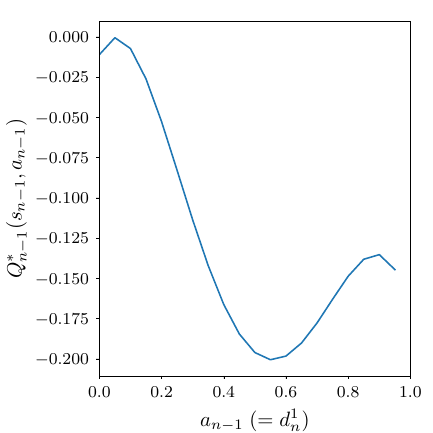}%
        \caption{Violation, non-incr.~utility.}%
        \label{fig:couterutilcost:a}%
    \end{subfigure}%
    \begin{subfigure}[b]{0.5\linewidth}%
        \centering%
        \includegraphics[width=\textwidth]{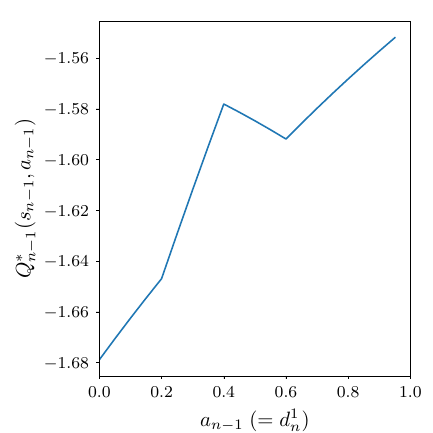}%
        \caption{Violation, non-conv.~cost.}%
        \label{fig:couterutilcost:b}%
    \end{subfigure}
    \caption{Violation of assumptions of Thm~\ref{thm:convex} results in bimodality of the optimal action-value function $Q^*(s,a)$: The figures show $Q^*(s,a)$ for a fixed stat $s$ and varying actions. Fig.~\ref{fig:couterutilcost:a} illustrates the violation of the assumption of increasing concave utility by a quadratic utility. Fig.~\ref{fig:couterutilcost:b} illustrates the violation of the assumption of convex costs by a non-convex cost.}
\label{fig:couterutilcost}
\end{figure}

\subsection{Non-convex costs}
\label{ap:nonconvexcost}
The only change we make to the setting of previous counterexample 
is that we fix the utility by the exponential utility $u(x) = -2e^{-\frac{1}{2}x}$,
which is concave increasing and further we deviate from convex
costs by setting $c(\delta^1_{k+1}-\delta^1_{k}, X_k) = \min\{ 0.25(\delta^1_{k+1}-\delta^1_{k}), 0.05  \}$, which is non-convex but still continuous cost.

The figure \ref{fig:couterutilcost:b} shows the optimal action-value function $Q^*(s_{n-1},a_{n-1})$ as function of action $a_{n-1}= \deval_{n}^1$
for the state $s_{n-1} = (n-1, \deval_{n-1},x_{n-1})$, where the cash is $\deval_{n-1}^0 = 0$, number of shares held is $\deval_{n-1}^1 = 0.4$ and the price is $x_{n-1}=1.0$.
The bi-modality in action $a_{n-1} = \deval_{n}^1$ is clearly apparent.

\subsection{Bounds on cash state variable}
\label{ap:boundsconstraint}
Here we demonstrate that imposing a convex constraint on the cache
state variable $\delta^{0}_k$ can lead to disconnected set of allowed actions.
The constraint is set in form of bounds $b_{min} \leq \delta^{0}_k \leq b_{max}$.
The key idea is that pre-image of convex set (here $[b_{min},b_{max}]$)
in concave map (here $g_k$ in \eqref{eq:chashupdate}) is generally
non convex. The bounds (especially $b_{min}$) are very natural to consider.

The setting for this counter example is the following.
We will consider the same exponential utility as in Sec.~\ref{ap:nonconvexcost}. We will consider two possible costs:
a quadratic cost (the same as in section \ref{ap:nonincutil})
and nonconvex continuous cost (the same as in section \ref{ap:nonconvexcost} only 10 times higher).
We will consider bit larger discretization grids.
The action space is formed by 40 levels uniformly spaced between 0 and 2,
i.e. $\mathcal{A} = \{0, 0.05, \ldots, 1.95 \}$.
We set the bounds on the cache random variable $\delta^0_k$ as $b_{min}=0$ and $b_{max}=4$. The price process and liability are the same as in previous counterexamples (cf. section \ref{ap:nonincutil}).

First we describe example with quadratic cost.
In Fig.~\ref{fig:boundsconstraint:a} there is $Q^*_{n-1}(s_{n-1},a_{n-1})$ plotted as color map
over its available actions (actions in the white regions are not allowed). The two dimensions spanning the plot are the action $a_{n-1} = \deval_{n}^1$ and cash component $\deval_{n-1}^{0}$ of the state. The other components of the state $s_{n-1} = (n-1,\deval_{n-1},x_{n-1})$ are
held fixed on values  $\deval_{n-1}^1 = 1.5$ and $x_{n-1} = 2$. One clearly see that this set is non-convex.
The figure \ref{fig:boundsconstraint:b} is a plot of $Q^*_{n-1}(s_{n-1},a_{n-1})$ as function of $a_{n-1}$ for the specific fixed state denoted by black horizontal line in \ref{fig:boundsconstraint:a}.
The gaps in action dimension correspond to not allowed actions,
i.e actions which leads cash state component $\deval^0_n$ being outside of the interval $[b_{min},b_{max}]$. One can see that the problems with the convex cost arises rather for cash $\deval_{n-1}^{0}$ near upper bound $b_{max}$.

Finally the example with nonconvex continuous cost is shown 
in the figure \ref{fig:boundsconstraint:c} and \ref{fig:boundsconstraint:d}, here the price component of the state was fixed at $x_{n-1} = 1$. The setting of other
parameters is the same as for quadratic cost. This eample was included
to demonstrate possible problems when cash $\deval_{n-1}^{0}$ is near lower bound $b_{min}$.

Notice that in original formulation the problem with bounding cash
variable $\delta^0_k$ does not arise,  because the accumulated
costs, which are causing the problem, are excluded from $\delta^0_k$
(and put to the last state component denoted $W_k$ in the section \ref{sec:portfolioReplAndConvexOpt}). This does not mean that the original formulation is
better, but rather simpler. Usually, when one wants to constraint/bound
the cash, it is desirable to do it including the accumulated
costs, especially when the costs could grow fast.

\begin{figure}[h]
    \begin{subfigure}[b]{0.49\linewidth}
        \centering%
        \includegraphics[width=\textwidth]{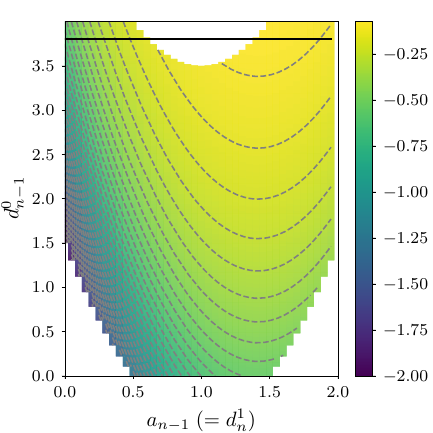}%
        \caption{Upper bound on cash.}%
        \label{fig:boundsconstraint:a}%
    \end{subfigure}\hfill%
    \begin{subfigure}[b]{0.49\linewidth}%
        \centering%
        \includegraphics[width=\textwidth]{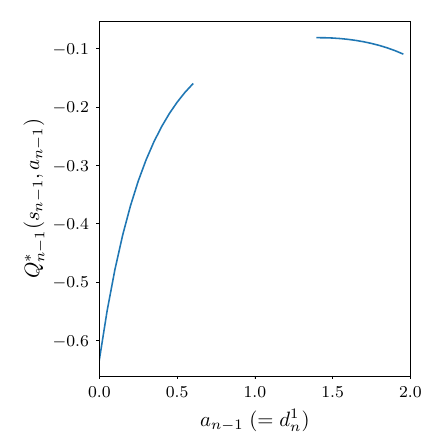}%
        \caption{Upper bound on cash.}%
        \label{fig:boundsconstraint:b}%
    \end{subfigure}\\
    \begin{subfigure}[b]{0.49\linewidth}%
        \centering
        \includegraphics[width=\textwidth]{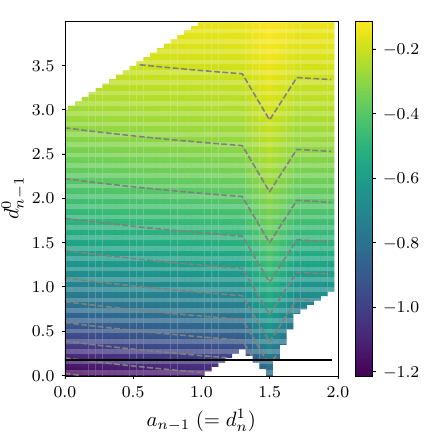}%
        \caption{Lower bounds on cash.}%
        \label{fig:boundsconstraint:c}%
    \end{subfigure}\hfill%
    \begin{subfigure}[b]{0.49\linewidth}%
        \centering%
        \includegraphics[width=\textwidth]{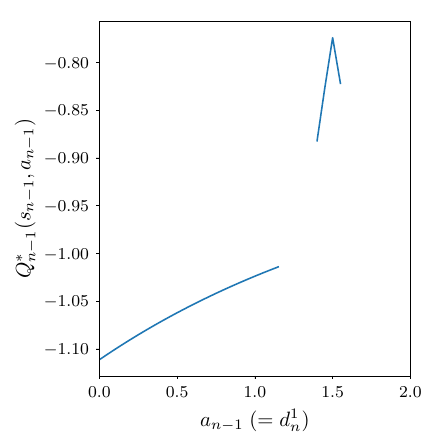}%
        \caption{Lower bounds on cash.}%
        \label{fig:boundsconstraint:d}%
    \end{subfigure}\\
   \caption{Plots of optimal action-value functions  $Q^{*}_{n-1}(s_{n-1},a_{n-1})$ illustrating problematic bounds on the cash state variable $\delta^0_k$. Fig.~\ref{fig:boundsconstraint:a} and Fig.~\ref{fig:boundsconstraint:b} illustrate quadratic costs and Fig.~\ref{fig:boundsconstraint:c} and Fig.~\ref{fig:boundsconstraint:d} illustrate non-convex costs. The non-convex sets of admissible actions (colored areas in heat maps) lead to disconnected sets of available actions in (b) and (d).}       
\label{fig:boundsconstraint}
\end{figure}

\section{Details on experiments}

\subsection{Experiments in minimal environments with non-convex rewards}

\subsubsection{Experimental Setup for Learning a Deterministic Sequence of Actions}
\label{app:learningDeterministic}
This section details the experimental configuration used in Sec.~\ref{se:learningDeterministic}. States are represented as in Equ.~\eqref{eq:simpleState} with $\adim=1$. Rewards are granted after the terminal action at $t_4$. Rewards only depend on correctness of the learned assignment $t_k\mapsto a_k$, such that the market and P\&L state components of~\eqref{eq:simpleState} play no role here. We use a dummy normal random variable to fill the market component.

For AlphaZero rewards are scaled to the interval $[-1, 1]$. A reward of $1$ signifies that all actions were taken correctly. Each neural UCT iteration consists of $5{,}000$ training episodes with $25$ individual Monte Carlo simulations per episode. The AlphaZero model is trained for $30$ cycles, totaling $150{,}000$ games. The neural network architecture consists of four hidden layers with $256$ neurons per layer, batch normalization, and ReLU activation functions. Training employs an Adam optimizer with a learning rate of $0.001$ for $10$ epochs with batch size $64$. The network predicts two outputs: value (tanh-activation) and action probabilities (softmax-activation). Validation is performed after each (of the $30$) training cycle on $10{,}000$ random market evolutions, with new networks accepted only upon performance improvement.

For DH the deterministic NN policies $F_k$ provide actions $F_k(s_k)=\delta_k^1$, where the time variable $t_k$ in $s_k$ is scaled to $[0,1]$. To compute \lq\lq{}rewards\rq\rq{} we employ squared distance loss between $F(s_k)$ and the correct action $a_k$. Since $a_k\in[-1,1]$ we employ $\tanh$ activation. As AlphaZero, DH is trained on a total of $150{,}000$ games, with episodes per epoch calculated by division through the number of epochs. The hyperparameters [hidden size, learning rate, number of layers, epochs] are optimized using a Tree-structured Parzen Estimator with the Optuna~\cite{optuna} package. For each parameter configuration, $100$ runs are executed, and the mean number of correct actions across these runs serves as the loss function to be maximized. After $50$ tuning rounds, the optimized parameters were determined to be: hidden size of 64, learning rate of 0.001, 2 layers, and 50 epochs. A comprehensive overview of the experimental configurations for both AlphaZero and Deep Hedging is provided in Table~\ref{tab:a1_table}. These experiments were conducted on a workstation running Ubuntu, equipped with an Intel Core i5-12600K CPU, 32GB of DDR5 RAM, and an NVIDIA RTX 3090 GPU.\typeout{textwidth=\the\textwidth}

\begin{table}[htbp]
\centering
\renewcommand{\arraystretch}{1.2} 
{\small
\begin{tabular}{>{\raggedright\arraybackslash}p{\dimexpr 0.21\linewidth-2\tabcolsep} p{\dimexpr 0.79\linewidth-2\tabcolsep}}
\hline
\specialrule{0pt}{2pt}{1pt} 
\multicolumn{2}{l}{\large \textbf{AlphaZero Configuration}} \\
\specialrule{0pt}{1pt}{2pt} 
\hline
\textbf{State representation} & Game state $s_k$ as described in Equ.~\eqref{eq:simpleState}\\
\hline
\textbf{Init.~state} & {$s_0=(0,(0,0),0,0)$} \\
\hline
\textbf{Action space} & Discrete action space with 21 actions: purchase or sale of up to one share in fractions of 1/10 or no transaction. \\
\hline
\textbf{Reward structure} & Rewards granted at $t=T$, calculated by the number of correct actions.\newline
Rewards are restricted to the interval $[-1, 1]$, where 1 signifies all actions are taken correctly. \\
\hline
\textbf{Training procedure} & Each neural UCT iteration consists of 5,000 episodes with 25 individual simulations per episode. 
The AlphaZero model is trained for 30 cycles, totaling $150{,}000$ hedging games. \\
\hline
\textbf{Neural network architecture} & Four hidden layers with 256 neurons per layer.\newline
Batch normalization and ReLU activation functions.\newline
Network predicts two outputs: value (tanh-activation) and action probabilities (softmax-activation). \\
\hline
\textbf{Optimization} & Training for 10 epochs with batch size 64.\newline
Adam optimizer with a learning rate of 0.001. \\
\hline
\textbf{Validation} & After neural network retraining: rewards measured on 10,000 random paths. 
New networks accepted only upon performance improvement. \\
\hline
\specialrule{0pt}{2pt}{1pt} 
\multicolumn{2}{l}{\large \textbf{Deep Hedging Configuration}} \\
\specialrule{0pt}{1pt}{2pt} 
\hline
\textbf{State representation} & Neural network input includes $(A_{k-1},X_k)$ with continuous state representation. \\
\hline
\textbf{Action space} & Continuous actions $\in [-1,1]$ (implemented via final tanh activation). \\
\hline
\textbf{Training procedure} & Trained on a total of 150,000 hedging games.\newline
Episodes per epoch calculated as $150{,}000$ divided by the number of epochs ($3{,}000$ episodes per epoch with 50 epochs). \\
\hline
\textbf{Optimization} & Training with batch size 32.\newline
Adam optimizer with a tuned learning rate. \\
\hline
\textbf{Loss function} & Calculated on squared distance with the correct action. \\
\hline
\textbf{Hyperpar-ameter optimization} & Parameters [hidden size, learning rate, number of layers, epochs] optimized using Tree-structured Parzen Estimator with Optuna~\cite{optuna}.\newline
100 runs executed for each parameter configuration.\newline
Mean number of correct actions over 100 runs used as the optimization objective. 
100 tuning rounds executed. \\
\hline
\textbf{Optimized parameters} & Hidden size: 64\newline
Learning rate: 0.001\newline
Number of layers: 2\newline
Epochs: 50 \\
\hline
\end{tabular}
}
\caption{Experimental configuration for the deterministic sequence learning task described in Sec.~\ref{se:learningDeterministic}.}
\label{tab:a1_table}
\end{table}

\subsubsection{Experimental Setup for learning to a choose a known portfolio composition}\label{app:learningKnown}
This section details the experimental configuration used in Sec.~\ref{se:learningKnown}. The environment and experimental setup largely follows that of Sec.~\ref{se:learningDeterministic} with modifications to incorporate market stochasticity. 
The key differences in hyperparameters are summarized in Table~\ref{tab:a2_table}. For AlphaZero, we increased the network capacity to 512 neurons per layer while maintaining the same architecture depth and optimization parameters. For Deep Hedging, hyperparameter optimization was conducted as in the previous experiment, resulting in a substantially larger model with increased hidden size, deeper architecture, and reduced learning rate.

\begin{table}[htbp]
\centering
\renewcommand{\arraystretch}{1.2}
{\small
\begin{tabular}{>{\raggedright\arraybackslash}p{\dimexpr 0.21\linewidth-2\tabcolsep} p{\dimexpr 0.395\linewidth-2\tabcolsep} p{\dimexpr 0.395\linewidth-2\tabcolsep}}
\specialrule{0pt}{2pt}{1pt} 
& \textbf{Sec.~\ref{se:learningDeterministic}} & \textbf{Sec.~\ref{se:learningKnown}} \\
\hline
\specialrule{0pt}{2pt}{1pt} 
\multicolumn{3}{l}{\large \textbf{AlphaZero Configuration Changes}} \\
\specialrule{0pt}{1pt}{2pt} 
\textbf{Hidden size} & 256 neurons per layer & 512 neurons per layer \\
\hline
\specialrule{0pt}{2pt}{1pt} 
\multicolumn{3}{l}{\large \textbf{Deep Hedging Configuration Changes}} \\
\specialrule{0pt}{1pt}{2pt} 
\textbf{Hidden size} & 64 & 512 \\
\textbf{Learning rate} & 0.001 & 0.0001 \\
\textbf{Number of layers} & 2 & 5 \\
\hline
\end{tabular}
}
\caption{Modified hyperparameters for the known portfolio composition learning task described in Sec.~\ref{se:learningKnown}, compared to those used in Sec.~\ref{se:learningDeterministic}.}
\label{tab:a2_table}
\end{table}

\subsection{Experiments in toy examples of market models with non-convex rewards}

\subsubsection{Experimental setup for portfolio replication in a trinomial market with non-convex costs}\label{app:toyMarketEnvs:trinomialAbsorbing}
This section details the experimental configuration used in Sec.~\ref{sec:toyMarketEnvs:trinomialAbsorbing}. The environment and overall structure largely follow the setting in Sec.~\ref{se:learningDeterministic}, with adaptations for the trinomial market dynamics and the presence of non-convex transaction costs. States are represented as in Equ.~\eqref{eq:simpleState}, with the initial state being $0.4$ shares held, a $0.0$ cash balance and an initial share price $X_0=5$.
As before, rewards are granted after the terminal action at $t_{4}$.
For AlphaZero, the reward function is modeled after Equ.~\eqref{eq:targetProblem}, with rewards scaled to the interval $[-1, 1]$, where a reward of $1$ corresponds to a perfect hedge. Compared to Sec.~\ref{se:learningDeterministic}, we increased the neural network hidden layers size to $512$ neurons, increased the number of self-play training cycles, and reduced the number of games per cycle. All other architectural components and training procedures (optimizer, activations, validation) remain equal with those described in App.~\ref{app:learningDeterministic}.

For DH, the loss is also modeled by Equ.~\eqref{eq:targetProblem}, using the squared distance between predicted and optimal hedges. Hyperparameter optimization was performed using the same Tree-structured Parzen Estimator framework as before. The optimization process yielded a deeper architecture with more layers, a larger hidden size, and a reduced learning rate compared to the previous experiment. These changes reflect the increased complexity introduced by the non-convex cost structure. A summary of hyperparameters used in this setting is provided in Table~\ref{tab:a3_table}.

\begin{table}[htbp]
\centering
\renewcommand{\arraystretch}{1.2}
{\small
\begin{tabular}{>{\raggedright\arraybackslash}
p{\dimexpr 0.21\linewidth-2\tabcolsep} p{\dimexpr 0.395\linewidth-2\tabcolsep} p{\dimexpr 0.395\linewidth-2\tabcolsep}}
& \textbf{Sec.~\ref{se:learningDeterministic}} & \textbf{Sec.~\ref{sec:toyMarketEnvs:trinomialAbsorbing}} \\
\hline
\specialrule{0pt}{2pt}{1pt} 
\multicolumn{3}{l}{\large \textbf{AlphaZero Configuration Changes}} \\
\specialrule{0pt}{2pt}{1pt} 
\textbf{Action space} & 21 discrete actions & 20 discrete actions \\
\textbf{Training procedure} & 30 self-play cycles of $5{,}000$ episodes each & 60 self-play cycles of $2{,}500$ episodes each \\
\textbf{Hidden size} & 256 neurons per layer & 512 neurons per layer \\
\hline
\specialrule{0pt}{2pt}{1pt} 
\multicolumn{3}{l}{\large \textbf{Deep Hedging Configuration Changes}} \\
\specialrule{0pt}{1pt}{2pt} 
\textbf{Hidden size} & 64 & 128 \\
\textbf{Learning rate} & 0.001 & 0.0001 \\
\textbf{Number of layers} & 2 & 5 \\
\hline
\textbf{Init.~state} & {$s_0=(0,(0,0),0,0)$} & {$s_0 = (0,(0,0.4),5,0)$}\\
\hline
\end{tabular}
}
\caption{Modified hyperparameters for the trinomial market with non-convex costs learning task described in Sec.~\ref{sec:toyMarketEnvs:trinomialAbsorbing}, compared to those used in Sec.~\ref{se:learningDeterministic}.}
\label{tab:a3_table}
\end{table}

\subsubsection{Experimental setup for portfolio replication in a GBM market with non-convex costs} \label{app:toyMarketEnvs:gbmCosts}
This section details the experimental configuration used in Sec.~\ref{sec:toyMarketEnvs:gbmCosts}.
The environment and experimental setup largely follows that of Sec.~\ref{se:learningDeterministic} with  modifications to fit with the increased state-space. The differences in hyperparameters are summarized in Table~\ref{tab:a4_table}. {As in Sec.~\ref{sec:toyMarketEnvs:trinomialAbsorbing} states are represented as in Equ.~\eqref{eq:simpleState}, with the initial state as in App.~\ref{app:toyMarketEnvs:trinomialAbsorbing}. Prior to executing tree search market states are discretized by rounding to the second decimal place. Rewards are granted after the terminal action at $t_{4}$.}
For AlphaZero, the reward function is modeled by Equ.~\eqref{eq:targetProblem}, consistent with the formulation described in App.~\ref{app:toyMarketEnvs:trinomialAbsorbing}. In this setting, we increased the MCTS search iterations per episode and increased the training epochs, while maintaining the same architecture depth and optimization parameters. 

For Deep Hedging, the loss function is based on Equ.~\eqref{eq:targetProblem}, as in App.~\ref{app:toyMarketEnvs:trinomialAbsorbing}. Hyperparameter optimization was conducted as in the previous experiment, resulting in a larger model with increased hidden size, deeper architecture, and substantially increased learning rate.

\begin{table}[htbp]
\centering
\renewcommand{\arraystretch}{1.2}
{\small
\begin{tabular}{>{\raggedright\arraybackslash}p{\dimexpr 0.21\linewidth-2\tabcolsep} p{\dimexpr 0.395\linewidth-2\tabcolsep} p{\dimexpr 0.395\linewidth-2\tabcolsep}}
& \textbf{Sec.~\ref{se:learningDeterministic}} & \textbf{Sec.~\ref{sec:toyMarketEnvs:gbmCosts}} \\
\hline
\specialrule{0pt}{2pt}{1pt} 
\multicolumn{3}{l}{\large \textbf{AlphaZero Configuration Changes}} \\
\specialrule{0pt}{1pt}{2pt} 
\textbf{Action space} & 21 discrete actions & 20 discrete actions \\
\textbf{Training procedure} & 25 UCT simulations per episode & 35 UCT simulations per episode \\
\textbf{Optimization} & Training for 10 epochs & Training for 20 epochs \\
\hline
\specialrule{0pt}{2pt}{1pt} 
\multicolumn{3}{l}{\large \textbf{Deep Hedging Configuration Changes}} \\
\specialrule{0pt}{1pt}{2pt} 
\textbf{Hidden size} & 64 & 128 \\
\textbf{Learning rate} & 0.001 & 0.030 \\
\textbf{Number of layers} & 2 & 4 \\
\hline
\textbf{Init.~state} & {$s_0=(0,(0,0),0,0)$} & {$s_0 = (0,(0,0.4),5,0)$}\\
\hline
\end{tabular}
}
\caption{Modified hyperparameters for the GBM market with non-convex costs learning task described in Sec.~\ref{sec:toyMarketEnvs:gbmCosts}, compared to those used in Sec.~\ref{se:learningDeterministic}.}
\label{tab:a4_table}
\end{table}

\subsubsection{Experimental setup for portfolio replication in a trinomial market with trading constraints} \label{app:toyMarketEnvs:trinomialWithConstraint}
This section details the experimental configuration used in Sec.~\ref{sec:toyMarketEnvs:trinomialWithConstraint}.
This section details the experimental configuration used in Sec.\ref{sec:toyMarketEnvs:trinomialWithConstraint}. The environment and model setup largely follow that of Sec.\ref{se:learningDeterministic}, with modifications to accommodate trading constraints. {States are represented with Equ.~\eqref{eq:simpleState} as described previously, with the initial state having initial holdings of $1.5$ shares, cash balance of $0.8125$ and initial share price $X_0=5$. Rewards are granted after the terminal action at $t_{4}$.}
For this experiment, AlphaZero utilizes a reward based on exponential utility. The differences in hyperparameters are summarized in Table~\ref{tab:a5_table}. In this experiment, only the AlphaZero model was evaluated. The AlphaZero configuration remained consistent with prior experiments, except for the hidden layer size, which was increased to 512 neurons. Additionally, the action space was expanded to 40 discrete actions, linearly spaced in the interval $[0,2]$.

\begin{table}[tbp]
\centering
\renewcommand{\arraystretch}{1.2}
{\small
\begin{tabular}{>{\raggedright\arraybackslash}p{\dimexpr 0.21\linewidth-2\tabcolsep} p{\dimexpr 0.395\linewidth-2\tabcolsep} p{\dimexpr 0.395\linewidth-2\tabcolsep}}
& \textbf{Sec.~\ref{se:learningDeterministic}} & \textbf{Sec.~\ref{sec:toyMarketEnvs:trinomialWithConstraint}} \\
\hline
\specialrule{0pt}{2pt}{1pt} 
\multicolumn{3}{l}{\large \textbf{AlphaZero Configuration Changes}} \\
\specialrule{0pt}{1pt}{2pt} 
\textbf{Action space} & $21$ discrete actions & $40$ discrete actions \\
\textbf{Training procedure} & $25$ UCT simulations per episode, 30 self-play cycles of $5{,}000$ episodes each & $35$ UCT simulations per episode, 60 self-play cycles of $500$ episodes each \\
\textbf{Validation} & After neural network retraining: rewards measured on 10,000 random paths. & After neural network retraining: rewards measured on 500 random paths. \\
\hline
\textbf{Hidden size} & $256$ neurons per layer & $512$ neurons per 
layer \\
\hline
\textbf{Init.~state} & {$s_0=(0,(0,0),0,0)$} & {$s_0 = (0,(0.8125,1.5),5,0)$}\\
\hline
\end{tabular}
}
\caption{Modified hyperparameters for the experimental setup for portfolio replication in a trinomial market with trading constraints task described in Sec.\ref{sec:toyMarketEnvs:trinomialWithConstraint}, compared to those used in Sec.~\ref{se:learningDeterministic}.}
\label{tab:a5_table}
\end{table}

\subsection{Experiments for measuring sample efficiency}\label{app:sampleEfficiency}
This section details the experimental configuration used in Sec.~\ref{sec:sampleEfficiency}. A reservoir of $50{,}000$ price paths was generated using a trinomial market model with parameters $p_u \approx 0.247$, $p_d \approx 0.253$, and $p_m = 1 - p_u - p_d$. For each experiment, two non-overlapping subsets of the reservoir were sampled: one for training and one for evaluation.{As before, states are represented as in Equ.~\eqref{eq:simpleState}, with the initial state having initial holdings of $0$ shares, cash balance of $0.02783$ and initial share price $X_0=1$. Rewards are granted after the terminal action at $t_{19}$.}

The Deep Hedging agent was implemented as in prior experiments. For each run, it was trained from scratch on $20{,}000$ episodes,   using the training subset, with the same loss used for experiment described in Sec.~\ref{app:toyMarketEnvs:trinomialAbsorbing}. After training, performance was assessed on the evaluation subset by measuring terminal losses. The hyperparameters chosen for the Deep Hedging model are described in Table~\ref{tab:a7_table}.

The MuZero variant employed in these experiments incorporates a learned neural dynamics model. At the beginning of each training run, this dynamics model was trained to approximate the transition probabilities $p_u$, $p_d$, and $p_m$ conditioned on the current time step $t$ and stock price $X$. The network architecture consists of five linear layers with leaky ReLU activations and layer normalization. Training was conducted for $5{,}000$ epochs using the Adam optimizer (learning rate $0.001$) with a batch size of $32$, minimizing the Kullback–Leibler divergence between predicted transition probabilities and empirical frequencies in the training reservoir.

During training, MCTS simulations were performed using synthetic market trajectories sampled from the learned dynamics model. After each search, environment transitions were executed using the true transition probabilities from the reservoir. Evaluation of the MuZero policy was performed without search, using only the policy network. For MuZero, the reward function is modeled after Equ.~\eqref{eq:targetProblem}, consistent with the formulation described in App.~\ref{app:toyMarketEnvs:trinomialAbsorbing}.
Detailed hyperparameters for the MuZero model are displayed in Table~\ref{tab:a6_table}.

\begin{table}[htbp]
\centering
\renewcommand{\arraystretch}{1.2}
{\small
\begin{tabular}{>{\raggedright\arraybackslash}p{\dimexpr 0.21\linewidth-2\tabcolsep} p{\dimexpr 0.395\linewidth-2\tabcolsep} p{\dimexpr 0.395\linewidth-2\tabcolsep}}
& \textbf{Sec.~\ref{se:learningDeterministic}} & \textbf{Sec.~\ref{sec:sampleEfficiency}} \\
\hline
\specialrule{0pt}{2pt}{1pt} 
\multicolumn{3}{l}{\large \textbf{Deep Hedging Configuration Changes}} \\
\specialrule{0pt}{1pt}{2pt} 
\textbf{Training procedure} & Trained on a total of $150{,}000$ hedging games. Episodes per epoch calculated as $150{,}000$ divided by the number of epochs & Trained on a total of $20{,}000$ hedging games. \\
\hline
\end{tabular}
}
\caption{Modified Deep Hedging hyperparameters for the sample efficiency task described in Sec.\ref{sec:sampleEfficiency}, compared to those used in Sec.~\ref{se:learningDeterministic}.}
\label{tab:a7_table}
\end{table}

\begin{table}[htbp]
\centering
\renewcommand{\arraystretch}{1.2} 
{\small
\begin{tabular}{>{\raggedright\arraybackslash}p{\dimexpr 0.21\linewidth-2\tabcolsep} p{\dimexpr 0.79\linewidth-2\tabcolsep}}
\hline
\specialrule{0pt}{2pt}{1pt} 
\multicolumn{2}{l}{\large \textbf{MuZero Configuration}} \\
\specialrule{0pt}{1pt}{2pt} 
\hline
\textbf{State representation} & Game state $s_k$ as described in Equ.~\eqref{eq:simpleState}. {Initial state $s_0 = (0,(0.02783,0),1,0)$}\\
\hline
\textbf{Action space} & Discrete action space with 21 actions: purchase or sale of up to one share in fractions of 1/10 or no transaction. \\
\hline
\textbf{Reward structure} & Rewards granted at $t=T$, calculated based on the squared distance.\newline
Rewards are restricted to the interval $[-1, 1]$, where 1 signifies a perfect hedge. \\
\hline
\textbf{Training procedure} & Each neural UCT iteration consists of 500 episodes with 25 individual simulations per episode.\newline
The AlphaZero model is trained for 40 cycles, totaling 20,000 hedging games. \\
\hline
\textbf{Neural network architecture} & Four hidden layers with 512 neurons per layer.\newline
Batch normalization and ReLU activation functions.\newline
Network predicts two outputs: value (tanh-activation) and action probabilities (softmax-activation). \\
\hline
\textbf{Optimization} & Training for 10 epochs with batch size 64.\newline
Adam optimizer with a learning rate of 0.001. \\
\hline
\textbf{Validation} & After neural network retraining: rewards measured on 10,000 random paths.\newline
New networks accepted only upon performance improvement. \\
\hline
\textbf{Dynamics Model} & Neural Network with five linear layers of size $512$, leaky ReLU activation and layer normalization. \newline 
The dynamics model has two inputs, one for current time step and one for current stock price.
\newline 
The dynamics model is trained once at the beginning of every execution, for 5,000 epochs using the Adam optimizer, learning rate of 0.001, minimizing the Kullback–Leibler divergence between predicted transition probabilities $p_u$, $p_d$, $p_m$ and empirical probabilities measured from the training reservoir.
 \\
\hline
\textbf{Updated MCTS search} & During each UCT iteration, the state inside the MCTS tree is constructed using the probability distribution outputted by the dynamics model. \\
\hline
\end{tabular}
}
\caption{Detailed experimental configuration for the MuZero model utilized in the sample efficiency task described in Sec.~\ref{sec:sampleEfficiency}.}
\label{tab:a6_table}
\end{table}

\end{document}